\documentclass[10pt]{article}
\usepackage[preprint]{tmlr}
\usepackage{amsmath,amssymb,amsthm,booktabs,tabularx,array,graphicx,float,longtable}
\usepackage{microtype,xurl}
\usepackage[hidelinks]{hyperref}
\newtheorem{proposition}{Proposition}
\hypersetup{pdfauthor={Yupeng Yao},pdfsubject={Preprint draft}}
\hypersetup{pdftitle={Evaluating Trajectory Features for Routing Final-Layer Attention}}
\title{Evaluating Trajectory Features for\\Routing Final-Layer Attention}
\author{Yupeng Yao}\begin{document}\maketitle
\begingroup
\renewcommand{\thefootnote}{}
\footnotetext{Generative AI tools were used to assist with information retrieval and language polishing during manuscript preparation.}
\addtocounter{footnote}{-1}
\endgroup
\begin{abstract}
Attention routing requires a signal that predicts the value of attention on the current prefix. We evaluate whether hidden-state extrapolation error, curvature and error change improve this prediction beyond uncertainty, one-step displacement, position and state projections. Paired executions of the final attention layer supply signed next-token loss differences in frozen SmolLM3-3B-Base and Qwen3.5-4B-Base checkpoints. Utility-supervised routers are tested on 100 held-out PG-19 books at an identical causal 20\% invocation quota. None of six prespecified comparisons shows a positive gain after familywise correction. In Qwen3.5, a parameter-matched fixed-projection control lowers NLL by 0.00356 nats/token relative to the trajectory router (95\% interval 0.00218--0.00487). Secondary results depend on the operation removed, feature location and scoring horizon; frozen thresholds also drift substantially at longer horizons. Actual selected-query execution yields small long-sequence latency reductions with increased NLL, while learned routers remain slower during cached continuation. The study identifies limits on the incremental value of these trajectory summaries and separates allocation quality from measured inference benefit.
\end{abstract}

\section{Introduction}
Attention is useful when its output improves a prediction, but predictive uncertainty alone does not identify such positions. A model can remain uncertain after reading the available context, or confidently predict the wrong continuation before retrieving a relevant fact. A routing signal should therefore be evaluated against the change produced by a specified operation on the same prefix. Comparing signals under different supervision or invocation rates makes that question harder to answer.

Recent hidden-state trajectories are a plausible routing signal. Extrapolation error measures departure from a line fitted to preceding representations, while curvature measures a change in their direction. Associations between trajectory error and human reading costs beyond observed-word surprisal motivate this hypothesis \citep{barenholtz2026trajectory}. They do not determine whether the same statistics predict attention's effect on language-model loss. A useful comparison must also distinguish their contribution from simpler displacement measures and information in the current state.

We test this incremental contribution with paired final-layer interventions in two public base checkpoints. Every router learns from the same signed attention utilities and uses the same invocation quota. Scalar and fixed-projection controls match the trajectory router's head capacity; a second comparison adds trajectory summaries to a learned state projection. The primary experiment was specified and the fits frozen before modern test utilities were collected. It finds no positive gain in the six familywise-corrected comparisons. Secondary experiments examine why results change with the utility target, context length and feature location. Finally, selected-query measurements assess the quality and latency produced by the executed policies, including routing overhead. Our contribution is an empirical test of routing information under controlled interventions and budgets, with the full failure record retained.

\section{Related work}
\paragraph{Adaptive computation.}
Confidence-based adaptive language modelling varies decoding depth, while Mixture-of-Depths learns which tokens participate in a layer \citep{schuster2022calm,raposo2024mixture}. AMOR uses entropy to invoke attention in recurrent--attention hybrids \citep{zheng2026amor}. These methods establish different ways to allocate computation. We hold the backbone and optional operation fixed to compare the information supplied to a utility-supervised router. This experiment evaluates feature value within that setting, rather than comparing the end-to-end architectures.

\paragraph{Context and memory selection.}
Compressive Transformer studies memory for long-range modelling, and H$_2$O and SnapKV select retained attention states using different importance signals \citep{rae2019compressive,zhang2023h2o,li2024snapkv}. Retention changes which states future queries can read. Our primary intervention instead enables or bypasses the final attention computation while keeping its inputs and available keys fixed. An auxiliary local32 intervention changes the direct reads of that final layer. Historical memory-selection experiments are reported in the appendix; they address a separate selection target.

\section{Paired attention utility and routing}
\label{sec:method}
\subsection{A fixed final-layer intervention}
For each teacher-forced prefix, the preceding trunk produces a contextualized incoming state $h_t$. Let $A_t$ denote the final attention residual, $M$ its tokenwise feed-forward network and $N_2$ the original post-attention normalization. The bypass and enabled branches are
\begin{equation}
z_t^{(0)}=h_t+M(N_2(h_t)),\qquad
z_t^{(1)}=h_t+A_t+M(N_2(h_t+A_t)).
\end{equation}
Both use the original output normalization and vocabulary head. For the same next target $y_t$, define $\ell_t^{(a)}=-\log p_t^{(a)}(y_t)$ and the signed utility
\begin{equation}
U_t=\ell_t^{(0)}-\ell_t^{(1)},\qquad
\mathcal L(g)=\mathcal L(0)-\frac{1}{N}\sum_{t=1}^N g_tU_t.
\label{eq:paired}
\end{equation}
Negative utilities are retained. The loss identity holds for binary gates on fixed prefixes because final-layer outputs do not change later queries' shared inputs or keys. It measures the optional final operation in an otherwise contextualized model. Figure~\ref{fig:design} summarizes the intervention and evaluation.

\begin{figure}[t]
\centering\includegraphics[width=\linewidth]{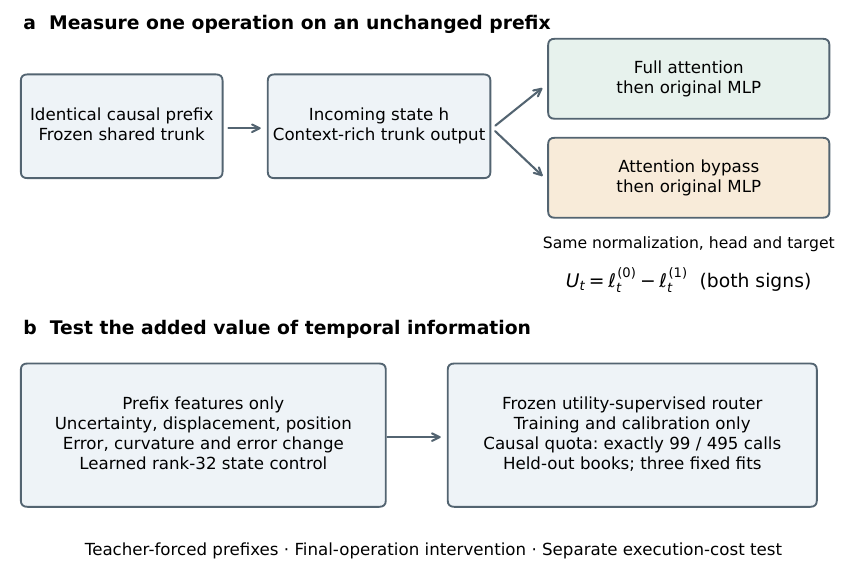}
\caption{Paired final-attention utility and controlled routing. Both branches receive the same preceding-trunk output and retain their corresponding final feed-forward computation. Primary features use the bypass state. Every compared router receives the same signed utility supervision. The prefix-only quota uses frozen calibration thresholds and allocates 99 calls among 495 scored positions. Incoming-state routing is a separate secondary experiment.}
\label{fig:design}
\end{figure}

\subsection{Features and matched controls}
Primary features are computed from $z_t^{(0)}$ before the next target is observed. Normalize the state as $\bar z_t=z_t/\sqrt{d^{-1}\|z_t\|^2+\epsilon}$, with $\epsilon=10^{-8}$. A line fitted to the preceding four states extrapolates $\widehat z_t=\sum_{j=1}^{4}\beta_j\bar z_{t-5+j}$, with $\beta=(-1/2,0,1/2,1)$. We use relative error $E_t=\|\bar z_t-\widehat z_t\|/(\|\bar z_t\|+\epsilon)$, one-step displacement $D_t=\|\bar z_t-\bar z_{t-1}\|/(\|\bar z_t\|+\epsilon)$, curvature $C_t=1-\cos_\epsilon(v_t,v_{t-1})$ for $v_t=\bar z_t-\bar z_{t-1}$, and error change $A_t=E_t-E_{t-1}$. Cosine uses independently epsilon-clamped norms. Entropy $H_t$ and top-two probability margin $M_t$ come from the bypass distribution; position $P_t$ is the zero-indexed input position divided by 512.

Table~\ref{tab:routers} defines the primary routers. C48, C\_RP3 and S32 already include displacement, so the tested increment is specifically $E_t,C_t,A_t$. C\_RP3 adds three fixed orthogonal current-state projections, with projection seed 127. S32 and S32\_T instead learn a biasless rank-32 projection of the entire normalized current state. Their projection sizes and initializations are paired, while their heads differ by 18 parameters. All routers use a linear--GELU--linear head and regress the same signed utility.

\begin{table}[t]
\centering
\caption{Primary feature comparisons. $X=(H_t,M_t,D_t,P_t)$ and $T=(E_t,C_t,A_t)$. The three comparisons per checkpoint are J32--C48, J32--C\_RP3 and S32\_T--S32. Fixed projections have no trainable weights.}
\label{tab:routers}
\begin{tabularx}{\linewidth}{lXrr}
\toprule
Router & Inputs & Head & Projection\\
\midrule
C48 & $X$ & 289 & 0\\
J32 & $X,T$ & 289 & 0\\
C\_RP3 & $X$ and three fixed state projections & 289 & 0\\
S32 & $X$ and learned rank-32 state projection & 1,331 & $32d$\\
S32\_T & $X,T$ and learned rank-32 state projection & 1,313 & $32d$\\
\bottomrule
\end{tabularx}
\end{table}

For each specification, seeds 901, 902 and 903 fit learning rates 0.0005 and 0.002 for 2,000 AdamW updates, with batch 1,024 and weight decay 0.01. Training-only feature and target scaling is frozen before calibration. Calibration loss selects checkpoints and a learning rate; calibration-score quantiles set thresholds. No test utilities enter this process.

\subsection{A common causal quota}
The primary policy allocates $k=\operatorname{round}(0.20N)$ calls within a known scored horizon $N$. With $q$ calls and $r$ positions remaining, it activates when $q=r$, rejects when $q=0$, and otherwise compares the current score with its frozen threshold. Updating $q\leftarrow q-g$ and $r\leftarrow r-1$ guarantees exactly $k$ calls. The rule uses no future score or target, but deadline forcing can select harmful positions. Unforced thresholds are evaluated separately with their actual invocation rates.

\section{Experimental protocol}
\paragraph{Models and data.}
We freeze SmolLM3-3B-Base and the text-only Qwen3.5-4B-Base model, containing 3,075,098,624 and 4,205,751,296 parameters, respectively \citep{smollm3basecard,qwen35basecard}. Their final layers use full attention; no backbone weights are fine-tuned. The PG-19 experiment uses 254 eligible sampled training books, all 50 calibration books and all 100 test books \citep{rae2019compressive,deepmindpg19}. Each tokenizer independently supplies two fixed 512-prediction chunks per book. Excluding the first 17 predictions leaves 495 scored positions per chunk, 99 calls at the primary quota and 99,000 test predictions per checkpoint. Exact revisions, input selection and exclusions appear in Appendix~\ref{app:details}.

\paragraph{Prospective comparisons.}
The protocol fixes features, fits, budgets and six primary contrasts before modern test-label collection. All four model--corpus fitting grids were frozen together. This is a local prospective protocol; the earlier small-model and GPT-2 work remains exploratory. One checkpoint per family and three router fits do not provide independent pretraining replications.

\paragraph{Inference.}
For each contrast, we average $(g_{\mathrm{candidate}}-g_{\mathrm{control}})U$ over the three fixed fits, then weight books equally. Positive gain favors the candidate. A 10,000-draw whole-book bootstrap, seed 217, preserves chunks and fits within a sampled book. We report individual 95\% percentile intervals, simultaneous Bonferroni intervals and two-sided null-centred bootstrap tests with Holm correction across all six primary contrasts. Inference conditions on the two fixed checkpoints and selected fits. Token-weighted descriptive NLL is kept separate from equal-book contrast inference. The full bootstrap definition and secondary analysis families are in the appendix.

\paragraph{Execution checks.}
Primary labels use the SDPA MATH reference after calibration fixtures exposed selected-row discrepancies under automatic SDPA in SmolLM3. The prespecified mean-absolute NLL tolerance remains 0.001 nats/token. Execution measurements include both the native automatic full model and matched MATH full. Saved failures and the numerical amendment are retained. Timing includes the actual feature, router, model and output-head operations; it does not infer speed from call counts.

\section{Results}
\subsection{No reliable gain from additional trajectory features}
Final attention reduced mean PG-19 NLL by 0.03813 nats/token in SmolLM3 and 0.16614 in Qwen3.5. Utility was positive on 48.04\% and 72.06\% of scored predictions, respectively. The negative tails in Figure~\ref{fig:operation} show why average attention benefit alone is insufficient to choose its invocation positions.

None of the six primary comparisons showed a positive gain after Holm correction (Table~\ref{tab:primary}). J32's small favorable means against C48 did not establish an improvement. In Qwen3.5, C\_RP3 achieved lower NLL than J32 despite matching its head capacity and input dimension. Adding $E_t,C_t,A_t$ to the learned state router also gave no reliable gain, with directions changing across fitting seeds. The simultaneous upper bounds for all six contrasts were below the prespecified 0.005-nats/token practical discussion threshold. These results bound large positive increments in the tested setting without establishing exact equivalence.

\begin{table}[t]
\centering
\caption{All six primary contrasts. Gains and individual 95\% whole-book intervals are in $10^{-3}$ nats/token; positive values favor the candidate. Each estimate weights 100 books equally and averages three fixed router fits. $P_H$ is Holm-adjusted over this complete family.}
\label{tab:primary}
\begin{tabular}{llrrr}
\toprule
Checkpoint & Candidate--control & Gain & 95\% interval & $P_H$\\
\midrule
SmolLM3 & J32--C48 & 0.437 & $[-0.138,1.046]$ & 0.602\\
 & J32--C\_RP3 & $-0.277$ & $[-0.997,0.428]$ & 0.843\\
 & S32\_T--S32 & $-0.335$ & $[-0.968,0.295]$ & 0.843\\
Qwen3.5 & J32--C48 & 0.856 & $[-0.020,1.736]$ & 0.282\\
 & J32--C\_RP3 & $-3.557$ & $[-4.874,-2.182]$ & 0.00060\\
 & S32\_T--S32 & 0.426 & $[-0.347,1.221]$ & 0.843\\
\bottomrule
\end{tabular}
\end{table}

S32 achieved lower descriptive NLL than J32 by 0.01327 nats/token in SmolLM3 and 0.01068 in Qwen3.5 at the same quota. Its larger projection budget prevents treating this comparison as capacity-matched evidence about trajectory features. It does show that scalar-only baselines can miss a useful current-state control. Raw high-error priority also performed worse than random quota allocation in both checkpoints; large extrapolation error by itself was not a successful routing rule.

\begin{figure}[t]
\centering\includegraphics[width=\linewidth]{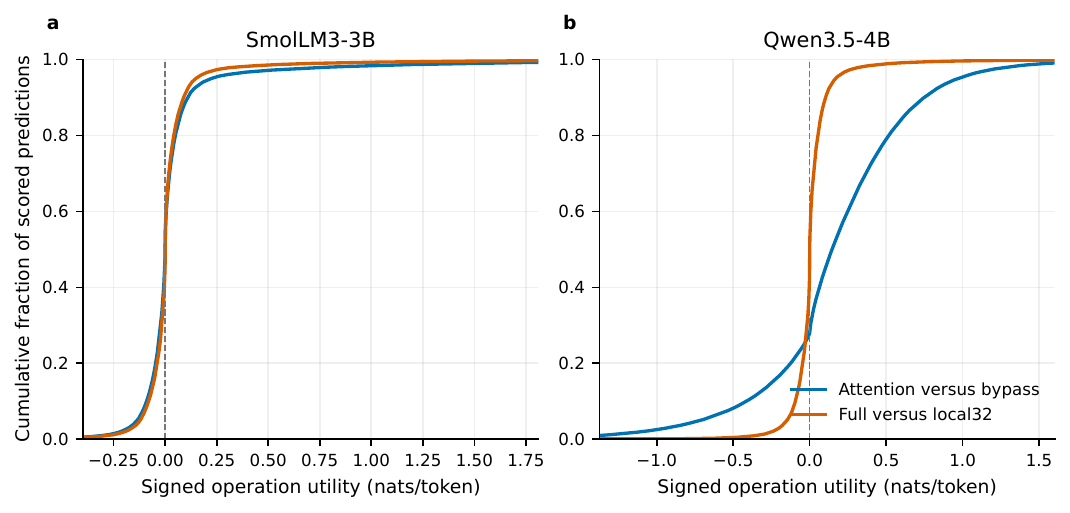}
\caption{Signed utility of the final attention operation. Each empirical distribution uses all 99,000 scored PG-19 predictions. Blue compares full attention with bypass; orange compares full with local32 final attention. Negative utilities are retained. Display limits clip the outer 0.5\% of pooled values per checkpoint; complete utilities are supplied in the review data. Earlier layers remain contextualized. Tokens describe the distributions; books are the inferential units.}
\label{fig:operation}
\end{figure}

\subsection{The utility target changes the comparison}
Restricting final attention to its current and preceding 31 positions gives a different target: $U_t^{\mathrm{remote}}=\ell_t^{(\mathrm{local32})}-\ell_t^{(1)}$. Full attention improved mean NLL over local32 by 0.01723 nats/token in SmolLM3 and 0.01032 in Qwen3.5. Local32 recovered 54.8\% and 93.8\% of the corresponding full-versus-bypass mean gain. This decomposition concerns direct reads by the final operation; its incoming state already contains contextual information.

Separately fitted routers for this target gave small, mixed trajectory increments. J32--C48 was 0.00062 nats/token in SmolLM3 and $-0.00018$ in Qwen3.5. The learned-state additions were 0.00009 and 0.00034. These secondary estimates, with full intervals in the appendix, show that removing final attention and withholding only its direct distant reads need separate evaluations.

\subsection{Thresholds and feature effects change at longer horizons}
At the nominal 20\% threshold, unforced J32 activation rose from 20.46\% to 86.86\% in SmolLM3 and from 18.71\% to 35.55\% in Qwen3.5 when scoring extended from 512 to 4,096 predictions (Figure~\ref{fig:boundaries}). Position remains divided by 512 and moves outside its fitting range. Such rate drift makes unforced NLL changes unsuitable as matched-budget evidence.

The quota-controlled shift also changed feature rankings. At 4,096 predictions, J32--C48 was $-0.00007$ nats/token in SmolLM3 and 0.00174 in Qwen3.5, while S32\_T--S32 was $-0.00127$ and $-0.00293$. The 99 eligible midpoint sequences come from the same test books and can overlap the core chunks. They are a horizon-shift analysis, not an independent book cohort.

Changing the feature location changed this long-context result. Incoming-state routers avoid computing the bypass feed-forward network and vocabulary head before gating. Adding trajectory summaries to their learned state control gave descriptive bypass/full gains of 0.00057 nats/token in SmolLM3 (unadjusted 95\% interval 0.00031--0.00083) and 0.00231 in Qwen3.5 (0.00171--0.00294). The same additions gave negative direct distant-read gains of $-0.00088$ and $-0.00065$. Favorable secondary means therefore depend on both feature location and the operation whose utility is predicted.

Separately refitted WikiText-2 routers supplied a second corpus check \citep{merity2017pointer}. J32--C48 and the learned-state additions had individual intervals spanning zero in both checkpoints. The reused article partitions and refitting protocol are documented in the appendix. All analyses in this subsection are secondary; they are kept separate from the six primary tests.

\begin{figure}[t]
\centering\includegraphics[width=\linewidth]{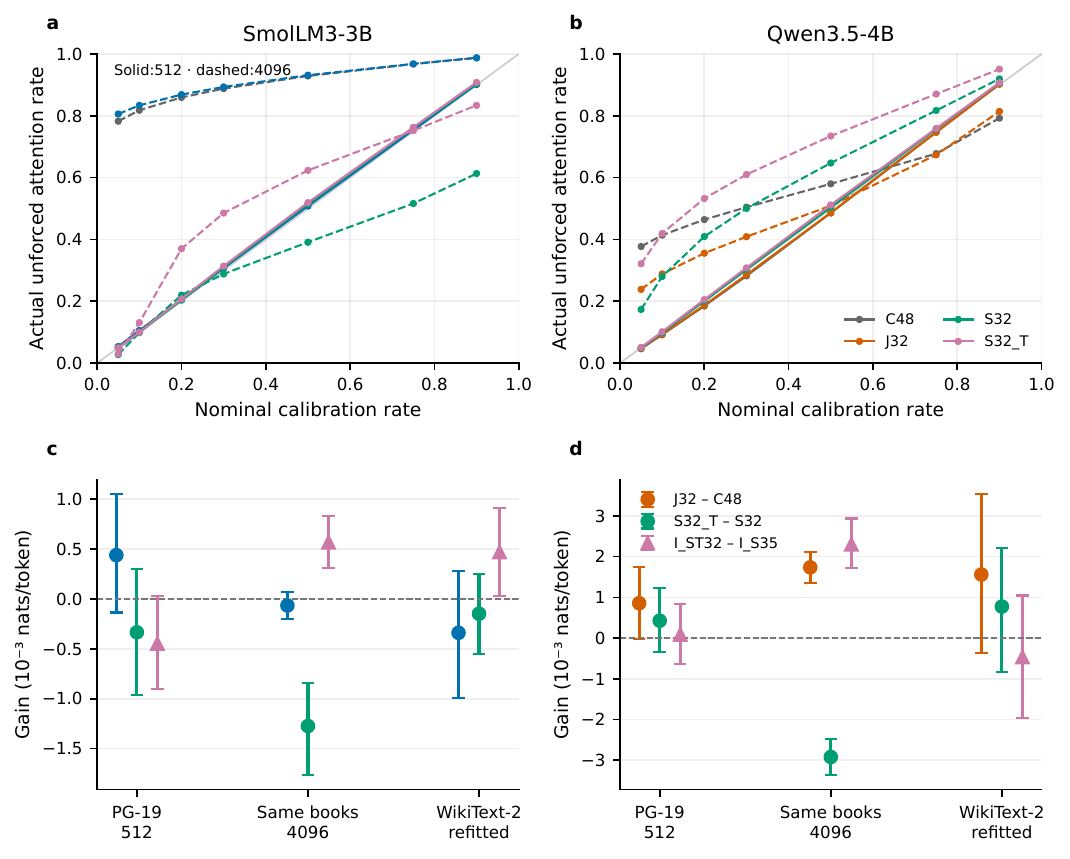}
\caption{Threshold drift and secondary feature comparisons. Top panels compare actual unforced invocation rates with their nominal calibration rates at 512 (solid) and 4,096 (dashed) predictions. Bottom panels show 20\%-quota contrasts, including incoming-state additions (triangles), with individual unadjusted 95\% book/article-bootstrap intervals. PG-19 has 100 core books and 99 eligible books in the overlapping horizon shift; WikiText-2 has 45 SmolLM3 and 46 Qwen3.5 test articles. Token NLL is not pooled across tokenizers.}
\label{fig:boundaries}
\end{figure}

\subsection{Selected-query execution trades quality for limited latency gains}
We measured scoring at 512, 1,024 and 4,096 predictions and 64 cached forced-input computations after 512- or 1,024-token prompts. Sequence-scoring conditions used five warmups; all conditions used 30 randomized paired blocks, with fitting seed 901. Cached continuation followed independent cache setup without a continuation warmup. Scoring quality was collected from each policy's executed forward pass on the full locked roster, separately from its 30 timing sources. Cached quality uses the timed continuation sources. Figure~\ref{fig:execution} pairs each latency with the loss produced by that policy and backend.

Incoming-state routers showed no consistent latency improvement at 512 and 1,024 predictions. Bypass-state J32 was slower at 512, with native-full/policy median-time ratios of 0.863 in SmolLM3 and 0.827 in Qwen3.5. This prototype computes a bypass feed-forward/head pass for features, then recomputes final feed-forward and head operations on all rows. It does not reuse the unselected-row outputs. The measured overhead therefore includes avoidable implementation work.

At 4,096 predictions, incoming-state I\_S35 had native-relative speed ratios of 1.043 in SmolLM3 (descriptive 95\% paired-block interval 1.013--1.060) and 1.029 in Qwen3.5 (1.024--1.050). Its executed NLL increased by 0.00925 and 0.11014 nats/token, respectively. In SmolLM3, matched MATH full already had a native-relative ratio of 1.030; I\_S35's additional ratio against MATH was 1.012. No quality-acceptance margin was prespecified.

All learned routers were slower than native full and had higher NLL in the four cached conditions. After 1,024-token prompts, J32 ratios were 0.898 in SmolLM3 and 0.894 in Qwen3.5. Cached setup, time to first token and the first prompt-to-continuation prediction are outside this measurement. Keys and values are still written at skipped query steps so later reads retain their history.

Of 100 planned policy--conditions, 99 passed preflight and supplied 2,970 timed executions. SmolLM3 I\_S35 at length 1,024 failed the unchanged 0.001 mean-absolute NLL tolerance with discrepancy 0.001096 and was not timed. Measurements used shared Windows desktop graphics; the installed Qwen stack used reference PyTorch convolution and linear-attention fallbacks. These conditions and the retained failures accompany the full timing records.

\begin{figure}[t]
\centering\includegraphics[width=\linewidth]{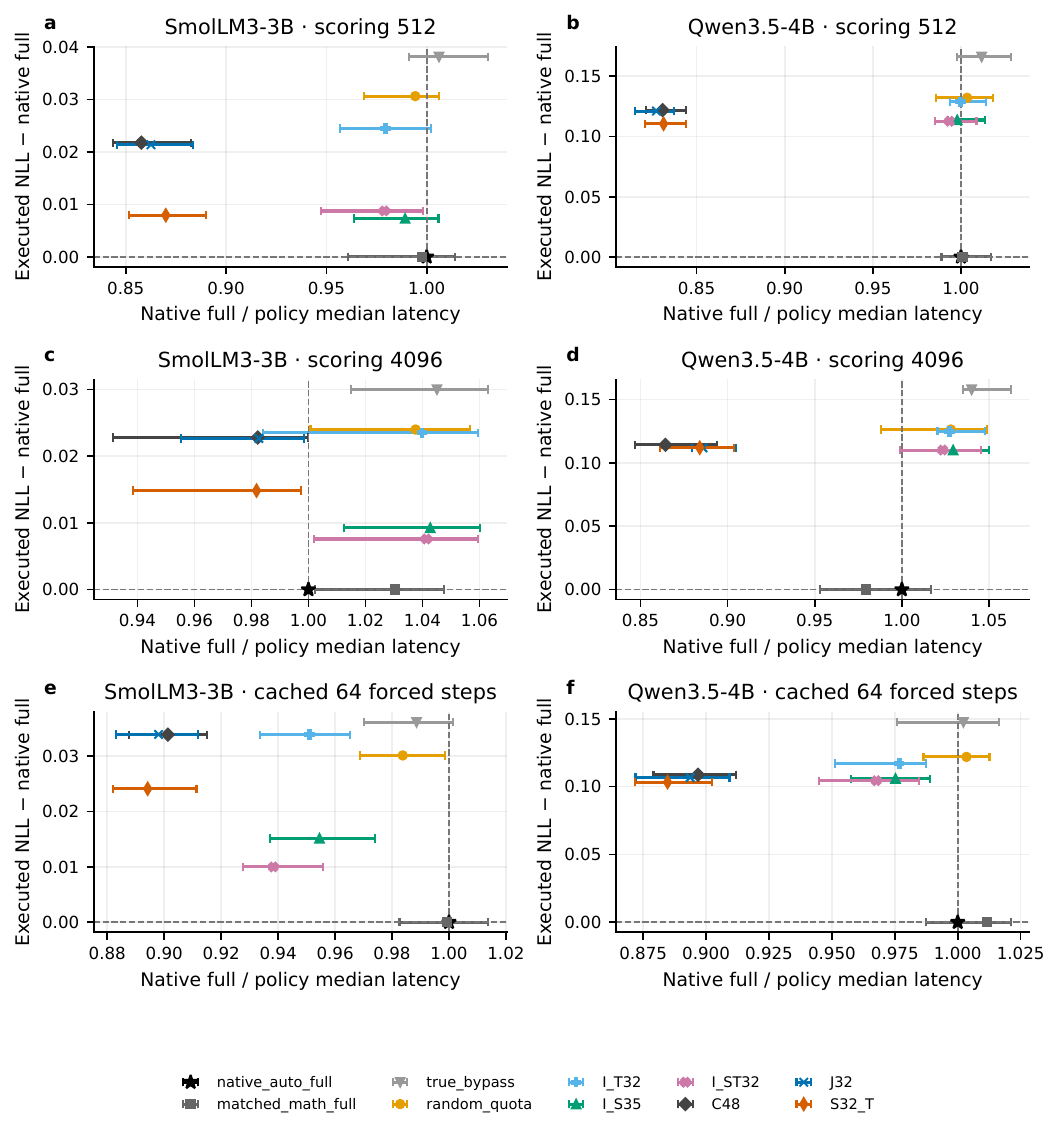}
\caption{Executed quality and latency. Columns are checkpoints; rows show scoring at 512 and 4,096 predictions and 64 cached forced-input computations after a 1,024-token prompt. Horizontal values divide native automatic full median latency by policy latency; values above one indicate faster execution. Intervals are descriptive, unadjusted 95\% paired-bootstrap intervals from 30 blocks. Vertical values are executed policy NLL minus native full NLL. Scoring quality uses the complete locked roster; cached quality uses the 30 timed sources. Matched MATH full separates backend changes from query omission. Routing overhead is included; cached setup and the first prompt-to-continuation prediction are excluded.}
\label{fig:execution}
\end{figure}

\section{Discussion}
The tested extrapolation and curvature summaries did not add a reliable primary benefit beyond displacement, uncertainty, position and state-projection controls. The fixed-projection comparison is particularly informative: in Qwen3.5, it improved allocation over J32 at the same head capacity. The learned-state comparison asks a complementary question about additional information under a larger common projection budget. Together, the comparisons show why an entropy-only baseline would give an incomplete assessment of a trajectory signal.

The results also make the utility target explicit. Direct distant reads by the final layer explain little of Qwen3.5's average full-versus-bypass gain, while its earlier layers remain contextualized. Different feature locations can give opposing secondary increments for bypass/full and local/full utilities. A routing feature therefore needs evidence for the operation it will control; a reading-time association or favorable result under another intervention cannot supply that evidence.

Several boundaries affect generalization. Rank-32 state projections and a fixed fitting recipe are limited learners; the primary results do not establish conditional independence of all trajectory information. One fixed checkpoint per family and an overlapping horizon-shift cohort cannot identify a population effect or isolate parameter scale from architecture and pretraining. The deadline quota controls invocation counts but assumes a known horizon. Its forced final decisions and the observed threshold drift leave unknown-length serving unresolved. Teacher-forced fixed prefixes also exclude feedback from generated tokens.

Execution adds a separate engineering constraint. Bypass features repeat work that an optimized implementation could reuse, and the preceding trunk remains active. The small long-sequence latency reductions came with higher NLL and partly reflected backend differences; cached continuation showed no learned-router speed benefit. Testing a broadly useful efficient architecture would require end-to-end quality and latency evidence with optimized kernels and an explicit quality criterion.

The appendix retains the earlier small-model, GPT-2 and engineered retention studies, including failed configurations and subsequent repairs. The retention instrument has at most 16 eligible VALUE states, making its nominal budget of 32 nonbinding. Its eligibility-aware result diagnoses how a budget is spent in that instrument; it does not establish an optimal future-memory selector. This supporting record is separate from the modern primary comparison.

\section*{Data and code availability}
PG-19 is available from \url{https://github.com/google-deepmind/pg19}; WikiText-2 is described by \citet{merity2017pointer}. The cited model cards identify the pinned checkpoints. The accompanying \texttt{CURVE\_LM\_reproduction\_lite.zip} supplies code, protocols, input manifests, primary paired-loss and score subsets, figure source data, per-sequence synthetic outcomes, benchmark timing records and compact reanalysis instructions. It distinguishes results that can be checked from recorded data from experiments requiring model reacquisition and fresh fitting. The complete internal archive additionally retains intermediate state caches and all candidate fits. No public repository accession or code DOI has been assigned. Third-party dependencies retain their original licences.

\clearpage\bibliographystyle{tmlr}\bibliography{references}
\clearpage
\appendix
\renewcommand{\thefigure}{A\arabic{figure}}\setcounter{figure}{0}
\renewcommand{\thetable}{A\arabic{table}}\setcounter{table}{0}
\renewcommand{\theHfigure}{appendix.\arabic{figure}}
\renewcommand{\theHtable}{appendix.\arabic{table}}
\section{Implementation and inference details}\label{app:details}

\paragraph{Study design and prospective boundary.}
The primary experiment tests whether temporal trajectory statistics improve allocation of final attention beyond controls that already contain one-step displacement, uncertainty, position and current-state projections. The data roster, features, fitting seeds, learning-rate grid, quota rule and six primary contrasts were locked before pretrained-model test utilities were collected. This was a local prospective protocol, not external preregistration. All four model--corpus fits, including scalers, selected checkpoints and calibration thresholds, were sealed on 6 October 2026 at 03:36:48 UTC. The seal records 240 selected controllers and 480 retained learning-rate candidates and verifies that no modern test label bank existed. Prior small-model and GPT-2 experiments remain historical exploratory evidence. Secondary WikiText-2 uses separately fitted routers and the earlier article partitions, so is not a fresh independent corpus. Failed and incomplete measurements remain identified in the run-status ledger.

\paragraph{Checkpoints and corpus.}
We use the text-only causal model from \texttt{Qwen/Qwen3.5-4B-Base}, revision \texttt{1001bb4d826a52d1f399e183466143f4da7b741b}, and \texttt{HuggingFaceTB/SmolLM3-3B-Base}, revision \texttt{d78a42f79198603e614095753484a04c10c2b940} \citep{qwen35basecard,smollm3basecard}. The loaded text models contain 4,205,751,296 and 3,075,098,624 parameters, respectively. Their final decoder layers, indices 31 and 35, contain full attention. No backbone weights are fine-tuned. These are one checkpoint per family; three router fits do not constitute independent pretraining replicates. 

PG-19 contains historical books published before 1919, with 28,602 training, 50 validation and 100 test books \citep{rae2019compressive,deepmindpg19}. From the numeric-ID-sorted official training roster, NumPy PCG64 seed 271828 selects 256 books without replacement. Input-only eligibility excludes training books 24610 and 31868 because fixed chunks overlap under at least one tokenizer. All 50 calibration and 100 test books remain eligible, with no replacements. Each tokenizer independently supplies two 513-token sequences per book at starts $\lfloor f(n-513)\rfloor$, $f\in\{0.25,0.75\}$, giving 512 predictions each. Excluding positions 0--16 leaves 495 scored predictions and exactly 99 calls per chunk at 20\%. The test endpoint has 99,000 scored positions in 100 book clusters. Book IDs are disjoint across partitions, and normalized-text hashes showed no overlap across partitions or with the earlier corpus. These checks do not establish semantic deduplication or absence from checkpoint pretraining.

WikiText-2 uses article-local, nonoverlapping 513-token sequences in the earlier partitions. PCG64 seed 917 samples at most 128 chunks per router-training, calibration and test role. Article counts and exact selected chunks are recorded separately for each tokenizer.

\paragraph{Paired final-operation intervention.}
The official preceding trunk runs once with teacher-forced inputs and \texttt{use\_cache=False}. We capture the incoming hidden state $h_t$ and the final decoder's original positional arguments, mask and keyword arguments. Write $A_t(h_{\leq t})$ for its attention residual and $M$ for its tokenwise feed-forward network. The final states are
\[
 z_t^{(0)}=h_t+M(N_2(h_t)),\qquad
 z_t^{(1)}=h_t+A_t+M(N_2(h_t+A_t)),
\]
where $N_2$ is the official post-attention normalization and attention includes its input normalization. The unchanged output normalization and vocabulary head yield $p_t^{(a)}$. Both branches retain their corresponding MLP; removing attention after the full MLP would define a different intervention.

Qwen3.5 replay preserves its split query/output-gate projection, query/key head normalization, rotary convention, grouped-query attention and sigmoid output gate. SmolLM3 replay preserves the captured layer-specific positional semantics; its final layer's executed \texttt{use\_rope} flag is false. Selected queries retain original token coordinates in their causal masks. Gating the final operation leaves no downstream stateful mixer whose recurrent state could differ between branches. Earlier globally or recurrently contextualized layers remain active.

For next target $y_t$, let $\ell_t^{(a)}=-\log p_t^{(a)}(y_t)$ and $U_t=\ell_t^{(0)}-\ell_t^{(1)}$. We retain negative utility. The term counterfactual refers to alternative internal executions on the same prefix, not observational identification of an unseen treatment outcome. For a binary gate and fixed prefixes,
\[
 \mathcal L(g)=\mathcal L(0)-N^{-1}\sum_{t=1}^N g_tU_t.
\]
This identity conditions on unchanged states and available keys. It does not identify a free-generation effect when changed tokens alter later prefixes.

\paragraph{Numerical integrity.}
Calibration-only fixtures check official full and bypass replay, future-token mutation invariance, binary endpoint identities, local/full row selection and actual selected-query equivalence at lengths 64 and 512. The prespecified mean absolute NLL discrepancy tolerance is 0.001 nats/token. The automatic SDPA path failed selected-row parity on SmolLM3 calibration inputs. The same SDPA MATH backend on official full and selected execution met this tolerance, so primary paired labels use that reference consistently. Reduced-precision FP16/BF16 GEMM reductions are disabled. Failed automatic-backend observations remain recorded as a pre-test numerical amendment. Head calculations use 128-token tiles with FP32 probability reductions. A speed comparison must also retain native automatic-backend full execution; the slower numerical reference alone is insufficient as a systems comparator.

\paragraph{Prefix features and feature location.}
Primary features use the bypass state $z_t^{(0)}$. With $\epsilon=10^{-8}$, define $\bar z_t=z_t/\sqrt{d^{-1}\|z_t\|^2+\epsilon}$. A least-squares line through the preceding four normalized states extrapolates
\[
 \widehat z_t=\sum_{j=1}^4\beta_j\bar z_{t-4+j-1},\qquad
 \beta_j=\tfrac14+\frac{6[j-5/2]}{12}.
\]
The current state is excluded from this fit. Features are relative extrapolation error $E_t=\|\bar z_t-\widehat z_t\|/(\|\bar z_t\|+\epsilon)$, relative displacement $D_t=\|\bar z_t-\bar z_{t-1}\|/(\|\bar z_t\|+\epsilon)$, curvature $C_t=1-\cos_\epsilon(v_t,v_{t-1})$, and error change $A_t=E_t-E_{t-1}$, with $v_t=\bar z_t-\bar z_{t-1}$. Cosine divides the dot product by independently epsilon-clamped vector norms. Zero velocity therefore gives curvature one. Position is the zero-indexed token index divided by 512. Entropy and top-two probability margin use the bypass distribution. Every decision uses information through the current input token; the next target is excluded.

A separately reported secondary feature-location experiment computes trajectory statistics and learned state projections from $h_t$, before final attention, feed-forward processing and the vocabulary head. Its three controllers are trajectory plus position (\texttt{I\_T32}), learned current state plus displacement and position (\texttt{I\_S35}), and that learned state plus error, curvature and error change (\texttt{I\_ST32}). These gates avoid a pre-gate vocabulary head. They use different information from the primary bypass-state controllers and are reported separately from the primary comparisons.

\paragraph{Matched routers and optimization.}
All primary routers regress identical signed utilities. \texttt{C48} uses entropy, margin, displacement and position with hidden width 48. \texttt{J32} adds error, curvature and error change with width 32. Both have 289 trainable head parameters. \texttt{C\_RP3} uses the four control scalars and three fixed orthogonal projections of the normalized current bypass state, also with 289 parameters. Its projection seed is 127. \texttt{S32} learns a biasless rank-32 projection of the complete normalized current state, combines it with the four control scalars, and uses width 35. \texttt{S32\_T} adds the three trajectory scalars with width 32. Their heads have 1,331 and 1,313 parameters, respectively, in addition to the same $32d$ projection weights. 

The network is linear--GELU--linear. Training-only scalar means, standard deviations and target scaling are frozen before calibration, with scale minima $10^{-6}$. Normalized complete states are cached as FP16 and promoted to FP32 for router fitting and evaluation. Fitting seeds are 901, 902 and 903. Paired state arms share projection initialization and minibatches but learn independently. Each seed fits learning rates 0.0005 and 0.002 with 2,000 AdamW updates, weight decay 0.01 and batch 1,024. A fitting-seed-initialized CPU PyTorch generator samples training positions with replacement. Calibration loss selects a checkpoint every 100 updates, then selects between the two candidates. All candidates and seeds are retained. Secondary \texttt{J32} objectives compare utility MSE with positive-utility binary cross-entropy and $\mathrm{KL}(p^{(1)}\Vert p^{(0)})$ regression. All ten specifications were fitted for both paired targets in each model--corpus combination.

\paragraph{Allocation and secondary context tests.}
Calibration-score quantiles give frozen thresholds. The primary prefix-only policy knows the scored horizon $N$ and allocates $k=\operatorname{round}(0.20N)$ calls. At each position, it rejects if remaining quota $q=0$, activates if $q$ equals remaining positions $r$, and otherwise activates when the current score exceeds its threshold. It then updates $q\leftarrow q-g$, $r\leftarrow r-1$. The invariant $0\leq q\leq r$ guarantees exactly $k$ calls. It is a causal deadline heuristic, not an optimal selector or a natural-EOS policy. Deadline forcing can select harmful positions. Raw entropy, error and displacement prioritize higher values as fixed before test; stable random allocation, always-bypass, always-full and a target-aware retrospective oracle are descriptive diagnostics. The oracle has no deployment or learnability interpretation.

The local-read control replaces only final attention's eligible keys with the current and previous 31 positions and defines $U_t^{\mathrm{remote}}=\ell_t^{(\mathrm{local32})}-\ell_t^{(1)}$. It uses separately fitted \texttt{C48}, \texttt{J32}, \texttt{S32} and \texttt{S32\_T} controllers, frozen before its test labels. To compare targets with feature information held fixed, these routers reuse the same bypass-state features, including bypass-distribution entropy and margin, rather than features from the local32 prediction. Incoming-state secondary controllers likewise retain their incoming features across targets. Preceding layers remain contextualized, so this measures marginal direct distant reads by the final operation. The local eligibility mask defines a semantic intervention; it is not a benchmark of physically sparse sliding-window execution. Unforced threshold gates report actual rates at nominal 10, 20 and 50\%. A fixed 4,096-prediction horizon shift uses midpoint 4,097-token sequences from the same test books; book 31974 is excluded for length, leaving 99. Frozen 512-trained features and thresholds are unchanged, including position values beyond their training range. This overlapping sample is a secondary horizon shift, not an independent long-context benchmark.

\paragraph{Statistical analysis.}
Books, or WikiText-2 articles, are corpus-uncertainty units; tokens, chunks and fits are not independent replicates. For each checkpoint, the primary response is the equal-book mean of $(g_{\mathrm{candidate}}-g_{\mathrm{control}})U$, averaged over the three fixed fits. The six contrasts are \texttt{J32--C48}, \texttt{J32--C\_RP3} and \texttt{S32\_T--S32}, separately in each family on PG-19. Whole-book bootstrap sampling retains every paired chunk, policy and fit within a book. We use 10,000 draws, seed 217, individual percentile 95\% intervals and conservative 99.167\% intervals for simultaneous six-contrast coverage. Two-sided null-centred bootstrap $P$ values are
\[
 P=\frac{1+\#\{b:|\widehat\Delta_b-\widehat\Delta|\geq|\widehat\Delta|\}}{10001},
\]
with Holm correction at familywise $\alpha=0.05$. This is approximate bootstrap inference. Unadjusted interval exclusion is insufficient for a confirmatory claim. Per-fit values, harmful selected utility and positive utility capture accompany the signed endpoint. Intervals condition on the fixed pretrained checkpoints and fitted policies; they do not quantify pretraining or model-family uncertainty. We prespecified 0.005 nats per model-specific token as a practical discussion threshold, motivated by earlier GPT-2 controlled gains, without claiming a power calculation. Fixed sample sizes are not expanded after inconclusive results. Secondary significance claims use within-family correction. Raw token NLLs from different tokenizers are not pooled.

\paragraph{Actual execution and quality.}
The execution protocol fixes fitting seed 901 and compares ten policies: native automatic-backend full, matched MATH full, true bypass, random 20\% quota, incoming-state \texttt{I\_T32}, \texttt{I\_S35} and \texttt{I\_ST32}, and bypass-state \texttt{C48}, \texttt{J32} and \texttt{S32\_T}. All policies other than the native automatic anchor use MATH. Actual selected queries retain the official positional, grouped-query and output-gating semantics. Rejected rows omit attention products and output projection. For uncached true bypass, the final attention and its projections are omitted. During cached execution, final-layer keys and values are written at every step, including skipped queries, so later reads do not lose history. The cached always-bypass continuation retains shared K/V-update cost despite never reading those keys; it is not the cheapest conceivable always-bypass cache implementation. Instrumented parity checks compare actual execution with the corresponding paired reference under the unchanged 0.001 mean absolute NLL tolerance.

For bypass-state controllers, the harness computes the pre-gate bypass MLP, normalization, vocabulary head and entropy/margin, then executes the gated final decoder and output head. The bypass feature MLP is not reused: the final decoder recomputes its MLP even on unselected rows. This avoidable implementation overhead is counted and disclosed rather than assumed absent. Incoming-state controllers use no pre-gate vocabulary head. Native full uses one decoder/head. The timed path includes every executed projection, router, feature calculation, state projection, MLP, normalization, tiled head, FP32 target-NLL reduction, cache operation, synchronization and dispatch. Source reading, tokenization and transfer of token IDs to the device are outside the timed boundary.

Executed quality is attempted from each verified policy's own forward pass on every row of the fixed scoring roster, separately from timing. At 512 predictions, the roster comprises both core chunks from each of 100 PG-19 test books. At 1,024 and 4,096 predictions, it comprises fixed prefix crops of midpoint stress-test sequences from 99 eligible books. Each sequence retains the common 17-position scoring exclusion. Saved losses, gates, book IDs and per-policy completion counts distinguish full-roster exports from incomplete measurements. Quality is not NLL mixed from dense labels or averaged over the 30 timing samples. The 1,024 and 4,096 sets overlap and are secondary evaluations. Numerical-reference quality is not transferred to native automatic execution or another feature-location policy.

\paragraph{Sequence scoring and cached continuation timing.}
The batch-one sequence-likelihood protocol covers 512-, 1,024- and 4,096-prediction forward passes and target losses, using the first 30 locked source rows. It specifies five untimed warmups and 30 paired blocks, with NumPy seed $731091+\mathrm{length}$ randomizing all ten policies and one common source row per block. Arms that fail parity, warmup, memory or execution remain recorded while other prespecified arms are attempted. Partial conditions are not complete ten-policy measurements. The first 17 positions receive no routed attention; nominal and physical query fractions are recorded separately.

Cached timing uses independent dynamic caches prepared from 512- or 1,024-token prompts, followed by 64 supplied input tokens. Cache setup computes the final attention query only at the last prompt position, writes the complete required key/value history, and performs no vocabulary-head evaluation. Setup time is measured separately and excluded from continuation latency. The first continuation input is supplied rather than generated by a scored prompt output: the 64 timed computations predict the next token after each supplied continuation input. The first prompt-to-continuation prediction is therefore excluded. This is neither time to first token nor complete prompt-plus-generation latency. Quota and state history advance once per supplied token, and the continuation endpoint uses its own fixed 64-position deadline without another warmup. Per-block executed continuation losses are retained alongside timing; they do not replace the complete-roster sequence-quality endpoint.

The original command-line benchmark stopped before parity or timing because Windows WDDM reported persistent desktop applications as compute/graphics processes. A platform addendum recorded before any systems measurement fixed a baseline PID inventory. The external orchestration runner imports all numerical operations from the unchanged sealed harness; it allows those baseline desktop PIDs and excludes a whole paired block if a new foreign GPU PID appears before or after it. Discarded blocks remain recorded, with at most five replacements. An unchanged PID does not imply an idle application, and unavailable telemetry remains unknown. These are measurements with one known experimental model process and shared desktop graphics, not verified device isolation. Raw telemetry, the original failed launch, external-runner hashes and per-arm failures are retained.

Synchronization brackets timing. Median and p95 wall-clock latency, tokens/s and peak allocated/reserved memory are reported for verified completed policies. The p95 concerns a complete sequence or 64-step continuation, not per-token latency. Ratios of median times use 10,000 paired block-bootstrap draws, seed 731092, and 95\% intervals against each available full reference. No quality-acceptance or noninferiority margin was prespecified. A ratio interval above one against native automatic full supports a measured latency reduction under the recorded shared-device conditions, but does not establish acceptable quality or general efficiency.

\paragraph{Software and hardware.}
The recorded environment is Windows with Python 3.12.14, PyTorch 2.11.0+cu128, CUDA 12.8, Transformers 5.18.0, tokenizers 0.23.2 and NumPy 2.5.3. The host uses an AMD Ryzen 9 9950X3D and NVIDIA RTX 5080. Checkpoint inference is BF16, with FP32 probability reductions and router parameters/optimization. Executed modelling files and intervention sources are identified by SHA-256 hashes. Only one pretrained checkpoint is resident at a time. Exact device/driver metadata and measurement conditions accompany the benchmark records.

Qwen3.5's recorded executions use the reference PyTorch fallbacks for \texttt{causal\_conv1d\_fn} and \texttt{chunk\_gated\_delta\_rule}, because the optional \texttt{causal-conv1d} and \texttt{flash-linear-attention} packages are absent. Paired arms share this installed stack. No package installation or numerical-source change was made in response to results. Absolute production performance and performance with optimal fused kernels remain untested.

\paragraph{Engineered retention support.}
Fresh generator seeds 730001 and 730002 supply 512 sequences at lengths 256 and 384 for the three frozen, previously repaired associative-recall checkpoints. Their first 128 sequences evaluate retention with recent 16 positions plus nominal budgets 16, 32 or 64. Eligible VALUE counts and selected identities are recorded separately. Eligibility-aware current utility ranks eligible observed states and fills unused nominal slots deterministically. With at most 16 eligible VALUE states, its nominal $B=32$ budget is nonbinding; eligible budgets 4, 8 and 16 are secondary. Accuracy and NLL remain distinct endpoints. Materialized-state masks measure selection quality rather than bounded online storage. Untouched sources do not erase earlier outcome-triggered architecture selection; its repairs are disclosed in Supplementary Information.

\section{Primary comparison intervals}
\begin{figure}[H]
\centering\includegraphics[width=\linewidth]{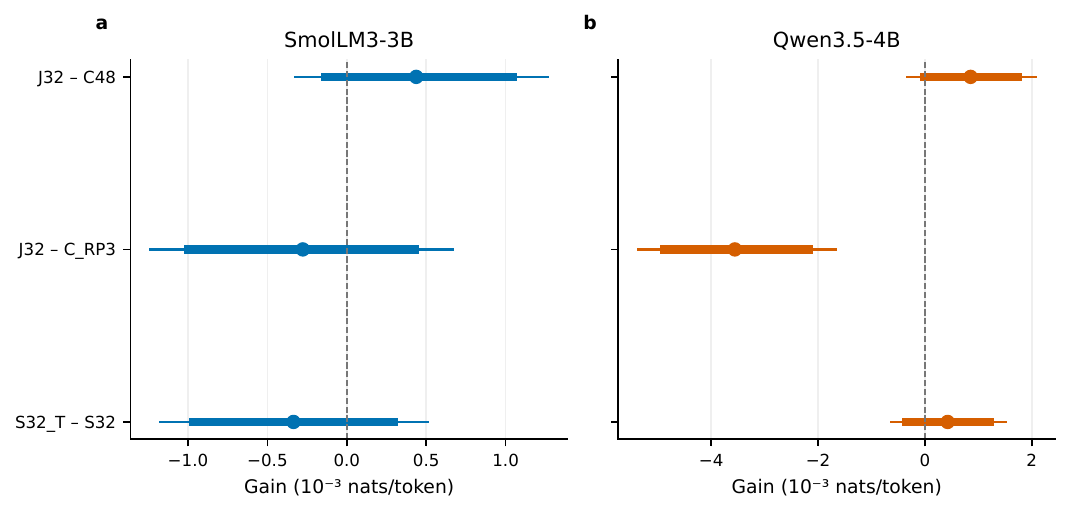}
\caption{Primary contrasts with individual 95\% intervals (thick) and simultaneous 99.167\% Bonferroni intervals (thin). Positive gains favor the additional-trajectory arm. The six-contrast family uses 100 book clusters per checkpoint and three fixed fits. C48 and C\_RP3 already contain one-step displacement.}
\label{fig:primary}
\end{figure}

\section{Supplementary theoretical details}

\paragraph{Information and conditioning.}
Let $\Phi$ contain information available to the router before the next target is observed. For an integrable signed operation utility $U$, define $\mu(\Phi)=\mathbb E[U\mid\Phi]$. A randomized gate has activation probability $q(\Phi)\in[0,1]$ and uses a coin independent of the target conditional on $\Phi$. The tower property then gives expected gain $\mathbb E[q(\Phi)U]=\mathbb E[q(\Phi)\mu(\Phi)]$. Neither an observed positive target utility nor a retrospective oracle establishes that $\mu$ is predictable from the available features.

\begin{proposition}[Expected upper-bound invocation budget]
\label{supp:budget}
For $0\leq\rho\leq1$, maximize $\mathbb E[q\mu]$ over measurable $q\in[0,1]$ satisfying $\mathbb E[q]\leq\rho$. If $\Pr(\mu>0)\leq\rho$, activating precisely the positive-utility set is optimal. Otherwise an optimal policy activates above a positive utility threshold, rejects below it, and randomizes independently on the threshold set to attain rate $\rho$.
\end{proposition}

\begin{proof}
At $\rho=0$, feasibility implies $q=0$ almost surely. If the positive-utility set has mass at most $\rho$, its indicator is feasible and pointwise maximizes $q\mu$. Otherwise select $\tau>0$ with $\Pr(\mu>\tau)\leq\rho\leq\Pr(\mu\geq\tau)$. This quantile exists for $0<\rho<\Pr(\mu>0)$. Let $q^*=1$ for $\mu>\tau$ and zero for $\mu<\tau$. When $\Pr(\mu=\tau)>0$, set
\[
 q^*=\frac{\rho-\Pr(\mu>\tau)}{\Pr(\mu=\tau)}
\]
on the equality set. When it has zero mass, the strict upper set already has mass $\rho$. For any feasible competitor, $(q-q^*)(\mu-\tau)\leq0$ pointwise. Hence
\[
 \mathbb E[(q-q^*)\mu]\leq
 \tau\{\mathbb E[q]-\mathbb E[q^*]\}
 =\tau\{\mathbb E[q]-\rho\}\leq0.
\]
Zero-utility activations can be added without benefit. Negative activations are unnecessary under an upper-bound budget.
\end{proof}

This result is an expected-rate statement. An exact finite quota may compel negative-utility activations, and a prefix-only deadline heuristic cannot consult the complete future score list. The primary policy is consequently not an implementation of the population optimum.

\begin{proposition}[Expected cost budget, sufficient conditions]
\label{supp:cost}
Let $c(\Phi)>0$ be measurable and integrable. Suppose $\lambda\geq0$ and a feasible $q_\lambda$ pointwise maximize $q(\mu-\lambda c)$ over $q\in[0,1]$, with complementary slackness $\lambda\{\mathbb E[q_\lambda c]-B\}=0$. Then $q_\lambda$ is optimal among policies satisfying $\mathbb E[qc]\leq B$.
\end{proposition}

\begin{proof}
Pointwise maximization gives
\[
 \mathbb E[(q-q_\lambda)(\mu-\lambda c)]\leq0.
\]
Thus
\[
 \mathbb E[(q-q_\lambda)\mu]\leq
 \lambda\{\mathbb E[qc]-\mathbb E[q_\lambda c]\}\leq0.
\]
For $\lambda=0$ the final inequality is immediate. For $\lambda>0$, complementary slackness and feasibility supply $\mathbb E[q_\lambda c]=B$ and $\mathbb E[qc]\leq B$.
\end{proof}

The policy activates when $\mu-\lambda c>0$, rejects when it is negative, and can randomize at equality. The proposition supplies sufficient conditions; it does not assert multiplier existence for every deterministic hard-budget problem. A FLOP estimate also does not identify a hard latency bound.

\begin{proposition}[Finite prediction regret]
\label{supp:regret}
For conditional utilities $\mu_i$ and predictions $r_i$, let $S^*$ choose at most $k$ largest positive true utilities and $\widehat S$ choose at most $k$ largest positive predictions. If $|r_i-\mu_i|\leq\varepsilon$ for every candidate, then
\[
 \sum_{i\in S^*}\mu_i-\sum_{i\in\widehat S}\mu_i
 \leq2k\varepsilon.
\]
\end{proposition}

\begin{proof}
Append $k$ dummy candidates with true and predicted utility zero. Select exactly $k$ candidates by true value and prediction, preferring dummies to real zero-valued candidates. Removing dummies recovers both positive, at-most-$k$ selectors. Let $A$ be the true selected set absent from the predicted selection, and $B$ the reverse difference, including dummies. Their cardinalities are equal and no greater than $k$. Pair their elements. For $i\in A$, $j\in B$, predicted top-$k$ selection implies $r_j\geq r_i$, giving
\[
 \mu_i-\mu_j\leq(r_i+\varepsilon)-(r_j-\varepsilon)
 \leq2\varepsilon.
\]
Sum over pairs. Dummies have zero error, and ties do not change the inequality.
\end{proof}

The assumption bounds conditional utilities uniformly. Neither MSE fitting nor empirical correlation establishes it. The result describes retrospective selection from a known candidate list, not the regret of the primary causal quota.

\paragraph{Oracle headroom without a learnable gate.}
Let the true target distribution be $(1/2,1/2)$, with bypass prediction $(0.6,0.4)$ and full prediction $(0.4,0.6)$. Realized utility is $-\log(1.5)$ for the first target and $+\log(1.5)$ for the second. Its conditional expectation is zero, so every prefix-only gate has zero expected gain. A target-aware oracle at rate $\rho\leq1/2$ activates with probability $2\rho$ on the second target and never on the first. It achieves expected reduction $\rho\log(1.5)$ despite using unavailable target information. A retrospective oracle therefore measures realized-label headroom rather than demonstrated causal learnability.

\paragraph{Prefix invariance and quota completion.}
For a declared horizon $N$, set $q_0=k$, $r_0=N$, with $0\leq k\leq N$. Reject when $q=0$, activate when $q=r>0$, and otherwise apply a fixed threshold to the current score. After rejection, $q'=q$ and $r'=r-1$; an interior rejection has $q<r$, so $q'\leq r'$. After activation, $q'=q-1\geq0$ and $r'=r-1$, preserving $q'\leq r'$. At the deadline $r_N=0$, this invariant implies $q_N=0$. The policy therefore activates exactly $k$ times. Earlier decisions depend only on prefix scores and previously determined quota state. Mutating future scores cannot change them. This proof requires a known horizon and gives neither optimal allocation nor unknown-EOS quota control. Strict threshold comparison, Python rounding, stable ordering and quota reset boundaries are recorded in the executable implementation.

\section{Supplementary reproducibility and evidence boundaries}

\paragraph{Historical evidence.}
The historical study includes 27 trained small checkpoints across the natural-language, dense-reference and synthetic configurations, plus one frozen public GPT-2 intervention. The first natural-language endpoint used retrospective exact-budget ranking. Later position, learned-state, balanced-position, causal-quota and systems checks are labelled according to their recorded exploratory timing. Outcomes from this record are not added to the six-contrast modern primary family. Training seeds and router-fitting seeds are distinct axes. Historical nats/byte and GPT-2 BPE-token quantities retain their original units and are not numerically pooled with modern-token losses.

\paragraph{Synthetic instrument and repair history.}
The engineered 40-token task contains early FACT/key/value declarations, delayed QUERY events, NOVEL events and random filler. The frozen support architecture uses a width-128 two-block local backbone, 0.02 positional scale, 6,000 training updates and value-category training weight 16. Two training-only untied 40-class heads supervise current and preceding observed tokens with CE weights 0.5 each. The preceding-token loss excludes the initial position. These 10,320 auxiliary parameters are removed at inference. The attention correction reads only observed VALUE tokens 16--23. Both auxiliary supervision and eligibility changed together after failed configurations were inspected; neither change receives a separate causal attribution. The reused earlier generator sources influenced architecture repair, making that historical support evidence exploratory.

Fresh sources retain the three frozen support checkpoints, model seeds 11, 23 and 37, and frozen routing rules. Query accuracy uses 512 independent sequences per length, whereas retention uses their first 128. The nominal $B=32$ endpoint compares current-utility retention with random, error, displacement and eligibility-aware current utility, for accuracy and NLL at both lengths. This produces 16 prespecified comparisons; sequence bootstrap uses 10,000 draws and seed 730021, with Holm adjustment within this support family. Three-checkpoint averages retain paired sequence identity, and intervals condition on the previously selected architecture and fixed weights.

At most 16 eligible VALUE states can occur in these sequences. An eligibility-aware nominal budget of 32 can retain all of them, so matching full-memory performance there does not establish optimal selection at a constrained eligible-state capacity. Secondary eligible budgets 4, 8 and 16 expose this distinction. Nominal slots and eligible states must remain separate columns in all tables. Accuracy and NLL can favor different policies: lower top-one accuracy alone is insufficient to call one retention policy worse in probabilistic prediction. All selected indices and counts are retained. These masks use materialized prefix states and do not measure physical memory reduction or bounded cache deployment.

\paragraph{Required run ledger.}
The reproducibility records distinguish completed, failed and omitted measurements. They retain model revisions, tokenizers, source hashes, data IDs, exact chunk starts, fitting candidates, selected calibration checkpoints, scaler hashes, thresholds, backend choices and collection manifests. The four model--corpus fits were sealed together before any modern test labels existed, with 240 selected controllers and 480 retained learning-rate candidates. Calibration-only repairs are separated from any repair after test observation. An outcome-triggered repair requires a new untouched test sample before a confirmatory interpretation. Failed selected-execution parity prevents transfer of dense-reference quality claims to that implementation. Synthetic repairs and systems conditions remain in the record even when they narrow the main conclusion.

\paragraph{Post-seal arithmetic verification.}
The original CPU checker encountered a two-ULP aggregate FP32 difference between directly selected branch losses and the utility-identity bookkeeping expression. Its source and failure were retained. An external reconciliation replayed the original FP32 reporting arithmetic, checked tokenwise rounding bounds and independently recomputed primary effects, intervals and multiplicity in FP64. Primary bootstrap $P$ values and Holm decisions were unchanged. No measurement, estimator, gate, threshold, endpoint or scientific tolerance was altered; the reconciliation applies only to diagnostic NLL bookkeeping.
\section{Complete pretrained-model allocation results}
All losses below use each checkpoint's tokenizer. They are not pooled across checkpoints. Three fixed router fits are averaged; book/article intervals condition on those fits and pretrained weights. These tables report the dense, consistently MATH-backed paired reference. Actual executed-policy results appear separately.
\begin{table}[H]\centering\small
\caption{Six prospective PG-19 primary contrasts at exactly 20\% calls. Gains and intervals are $10^{-3}$ nats/token. Individual 95\% intervals are descriptive; Holm-adjusted $P$ values control the six-contrast family. The machine-readable source also includes 99.167\% simultaneous intervals, all book values and each fitting seed.}
\begin{tabular}{llrr}\toprule Checkpoint & Contrast & Gain [95\% interval] & Adjusted $P$\\\midrule
SmolLM3 & J32--C48 & 0.437 [-0.138, 1.046] & 0.6019 \\
SmolLM3 & J32--C\_RP3 & -0.277 [-0.997, 0.428] & 0.8426 \\
SmolLM3 & S32\_T--S32 & -0.335 [-0.968, 0.295] & 0.8426 \\
Qwen3.5 & J32--C48 & 0.856 [-0.020, 1.736] & 0.2820 \\
Qwen3.5 & J32--C\_RP3 & -3.557 [-4.874, -2.182] & 0.0006 \\
Qwen3.5 & S32\_T--S32 & 0.426 [-0.347, 1.221] & 0.8426 \\
\bottomrule\end{tabular}\end{table}
\subsection{smollm3-3b-base / pg19\_512}
The sample contains 200 chunks, 100 book/article groups and 99,000 scored predictions. Full-operation NLL is 2.5433.
\begin{table}[H]\centering\small
\caption{Paired reference allocation at 20\% known-horizon quota. NLL/rate/harm are token-weighted and averaged over three fits; intervals for contrasts below instead weight books/articles equally. These estimands coincide for balanced PG-19 chunks but can differ for WikiText-2. Rate is exact within rounding. Harm retains the selected negative-utility sign.}
\begin{tabular}{llrrr}\toprule Target & Router & NLL & Rate & Selected harm\\\midrule
bypass & C48 & 2.5649 & 0.2000 & -0.0069 \\
bypass & C\_RP3 & 2.5642 & 0.2000 & -0.0069 \\
bypass & I\_S35 & 2.5512 & 0.2000 & -0.0090 \\
bypass & I\_ST32 & 2.5516 & 0.2000 & -0.0088 \\
bypass & I\_T32 & 2.5677 & 0.2000 & -0.0065 \\
bypass & J32 & 2.5645 & 0.2000 & -0.0069 \\
bypass & J\_KL & 2.5644 & 0.2000 & -0.0071 \\
bypass & J\_sign & 2.5686 & 0.2000 & -0.0063 \\
bypass & S32 & 2.5512 & 0.2000 & -0.0090 \\
bypass & S32\_T & 2.5516 & 0.2000 & -0.0092 \\
local & C48 & 2.5541 & 0.2000 & -0.0063 \\
local & C\_RP3 & 2.5536 & 0.2000 & -0.0064 \\
local & I\_S35 & 2.5487 & 0.2000 & -0.0075 \\
local & I\_ST32 & 2.5487 & 0.2000 & -0.0074 \\
local & I\_T32 & 2.5546 & 0.2000 & -0.0061 \\
local & J32 & 2.5535 & 0.2000 & -0.0064 \\
local & J\_KL & 2.5535 & 0.2000 & -0.0064 \\
local & J\_sign & 2.5551 & 0.2000 & -0.0053 \\
local & S32 & 2.5484 & 0.2000 & -0.0074 \\
local & S32\_T & 2.5483 & 0.2000 & -0.0074 \\
\bottomrule\end{tabular}\end{table}
\begin{table}[H]\centering\small
\caption{Secondary paired feature/objective contrasts at 20\% quota, $10^{-3}$ nats/token. These are unadjusted 95\% intervals; no additional confirmatory significance claim is made.}
\begin{tabular}{llr}\toprule Target & Contrast & Gain [95\% interval]\\\midrule
bypass & J32--C48 & 0.437 [-0.138, 1.046] \\
bypass & J32--C\_RP3 & -0.277 [-0.997, 0.428] \\
bypass & S32\_T--S32 & -0.335 [-0.968, 0.295] \\
bypass & J32--S32 & -13.268 [-14.735, -11.872] \\
bypass & I\_ST32--I\_S35 & -0.447 [-0.901, 0.027] \\
bypass & J32--J\_sign & 4.075 [3.148, 5.009] \\
bypass & J32--J\_KL & -0.043 [-0.414, 0.322] \\
local & J32--C48 & 0.615 [0.279, 0.953] \\
local & J32--C\_RP3 & 0.077 [-0.471, 0.614] \\
local & S32\_T--S32 & 0.088 [-0.250, 0.439] \\
local & J32--S32 & -5.131 [-6.032, -4.270] \\
local & I\_ST32--I\_S35 & 0.008 [-0.391, 0.401] \\
local & J32--J\_sign & 1.632 [1.069, 2.174] \\
local & J32--J\_KL & 0.027 [-0.271, 0.319] \\
\bottomrule\end{tabular}\end{table}
\subsection{smollm3-3b-base / wikitext2\_512}
The sample contains 128 chunks, 45 book/article groups and 63,360 scored predictions. Full-operation NLL is 2.1885.
\begin{table}[H]\centering\small
\caption{Paired reference allocation at 20\% known-horizon quota. NLL/rate/harm are token-weighted and averaged over three fits; intervals for contrasts below instead weight books/articles equally. These estimands coincide for balanced PG-19 chunks but can differ for WikiText-2. Rate is exact within rounding. Harm retains the selected negative-utility sign.}
\begin{tabular}{llrrr}\toprule Target & Router & NLL & Rate & Selected harm\\\midrule
bypass & C48 & 2.1966 & 0.2000 & -0.0068 \\
bypass & C\_RP3 & 2.1950 & 0.2000 & -0.0067 \\
bypass & I\_S35 & 2.1922 & 0.2000 & -0.0055 \\
bypass & I\_ST32 & 2.1916 & 0.2000 & -0.0056 \\
bypass & I\_T32 & 2.1973 & 0.2000 & -0.0053 \\
bypass & J32 & 2.1968 & 0.2000 & -0.0065 \\
bypass & J\_KL & 2.1972 & 0.2000 & -0.0071 \\
bypass & J\_sign & 2.1973 & 0.2000 & -0.0061 \\
bypass & S32 & 2.1922 & 0.2000 & -0.0055 \\
bypass & S32\_T & 2.1924 & 0.2000 & -0.0057 \\
local & C48 & 2.1929 & 0.2000 & -0.0065 \\
local & C\_RP3 & 2.1924 & 0.2000 & -0.0065 \\
local & I\_S35 & 2.1911 & 0.2000 & -0.0061 \\
local & I\_ST32 & 2.1910 & 0.2000 & -0.0060 \\
local & I\_T32 & 2.1927 & 0.2000 & -0.0050 \\
local & J32 & 2.1928 & 0.2000 & -0.0064 \\
local & J\_KL & 2.1929 & 0.2000 & -0.0076 \\
local & J\_sign & 2.1933 & 0.2000 & -0.0066 \\
local & S32 & 2.1911 & 0.2000 & -0.0064 \\
local & S32\_T & 2.1908 & 0.2000 & -0.0062 \\
\bottomrule\end{tabular}\end{table}
\begin{table}[H]\centering\small
\caption{Secondary paired feature/objective contrasts at 20\% quota, $10^{-3}$ nats/token. These are unadjusted 95\% intervals; no additional confirmatory significance claim is made.}
\begin{tabular}{llr}\toprule Target & Contrast & Gain [95\% interval]\\\midrule
bypass & J32--C48 & -0.340 [-0.992, 0.280] \\
bypass & J32--C\_RP3 & -1.544 [-2.448, -0.639] \\
bypass & S32\_T--S32 & -0.149 [-0.555, 0.248] \\
bypass & J32--S32 & -3.978 [-5.177, -2.866] \\
bypass & I\_ST32--I\_S35 & 0.469 [0.029, 0.911] \\
bypass & J32--J\_sign & 0.503 [-0.226, 1.239] \\
bypass & J32--J\_KL & 0.131 [-0.469, 0.708] \\
local & J32--C48 & -0.114 [-0.552, 0.303] \\
local & J32--C\_RP3 & -0.658 [-1.197, -0.132] \\
local & S32\_T--S32 & 0.323 [0.033, 0.631] \\
local & J32--S32 & -1.919 [-2.673, -1.160] \\
local & I\_ST32--I\_S35 & 0.239 [-0.058, 0.547] \\
local & J32--J\_sign & 0.249 [-0.339, 0.826] \\
local & J32--J\_KL & -0.311 [-0.962, 0.329] \\
\bottomrule\end{tabular}\end{table}
\subsection{smollm3-3b-base / pg19\_4096}
The sample contains 99 chunks, 99 book/article groups and 403,821 scored predictions. Full-operation NLL is 2.3695.
\begin{table}[H]\centering\small
\caption{Paired reference allocation at 20\% known-horizon quota. NLL/rate/harm are token-weighted and averaged over three fits; intervals for contrasts below instead weight books/articles equally. These estimands coincide for balanced PG-19 chunks but can differ for WikiText-2. Rate is exact within rounding. Harm retains the selected negative-utility sign.}
\begin{tabular}{llrrr}\toprule Target & Router & NLL & Rate & Selected harm\\\midrule
bypass & C48 & 2.3921 & 0.2000 & -0.0063 \\
bypass & C\_RP3 & 2.3920 & 0.2000 & -0.0062 \\
bypass & I\_S35 & 2.3787 & 0.2000 & -0.0084 \\
bypass & I\_ST32 & 2.3781 & 0.2000 & -0.0085 \\
bypass & I\_T32 & 2.3929 & 0.2000 & -0.0061 \\
bypass & J32 & 2.3922 & 0.2000 & -0.0063 \\
bypass & J\_KL & 2.3919 & 0.2000 & -0.0064 \\
bypass & J\_sign & 2.3923 & 0.2000 & -0.0062 \\
bypass & S32 & 2.3810 & 0.2000 & -0.0083 \\
bypass & S32\_T & 2.3823 & 0.2000 & -0.0084 \\
local & C48 & 2.3848 & 0.2000 & -0.0070 \\
local & C\_RP3 & 2.3849 & 0.2000 & -0.0065 \\
local & I\_S35 & 2.3771 & 0.2000 & -0.0083 \\
local & I\_ST32 & 2.3779 & 0.2000 & -0.0081 \\
local & I\_T32 & 2.3852 & 0.2000 & -0.0064 \\
local & J32 & 2.3850 & 0.2000 & -0.0065 \\
local & J\_KL & 2.3849 & 0.2000 & -0.0065 \\
local & J\_sign & 2.3852 & 0.2000 & -0.0063 \\
local & S32 & 2.3797 & 0.2000 & -0.0078 \\
local & S32\_T & 2.3799 & 0.2000 & -0.0076 \\
\bottomrule\end{tabular}\end{table}
\begin{table}[H]\centering\small
\caption{Secondary paired feature/objective contrasts at 20\% quota, $10^{-3}$ nats/token. These are unadjusted 95\% intervals; no additional confirmatory significance claim is made.}
\begin{tabular}{llr}\toprule Target & Contrast & Gain [95\% interval]\\\midrule
bypass & J32--C48 & -0.067 [-0.199, 0.065] \\
bypass & J32--C\_RP3 & -0.201 [-0.370, -0.036] \\
bypass & S32\_T--S32 & -1.274 [-1.768, -0.847] \\
bypass & J32--S32 & -11.151 [-12.275, -10.100] \\
bypass & I\_ST32--I\_S35 & 0.566 [0.308, 0.833] \\
bypass & J32--J\_sign & 0.132 [-0.374, 0.617] \\
bypass & J32--J\_KL & -0.296 [-0.487, -0.108] \\
local & J32--C48 & -0.145 [-0.314, 0.028] \\
local & J32--C\_RP3 & -0.123 [-0.290, 0.031] \\
local & S32\_T--S32 & -0.231 [-0.511, 0.031] \\
local & J32--S32 & -5.266 [-5.937, -4.647] \\
local & I\_ST32--I\_S35 & -0.879 [-1.207, -0.591] \\
local & J32--J\_sign & 0.266 [0.126, 0.400] \\
local & J32--J\_KL & -0.034 [-0.101, 0.027] \\
\bottomrule\end{tabular}\end{table}
\subsection{qwen3.5-4b-base / pg19\_512}
The sample contains 200 chunks, 100 book/article groups and 99,000 scored predictions. Full-operation NLL is 2.5703.
\begin{table}[H]\centering\small
\caption{Paired reference allocation at 20\% known-horizon quota. NLL/rate/harm are token-weighted and averaged over three fits; intervals for contrasts below instead weight books/articles equally. These estimands coincide for balanced PG-19 chunks but can differ for WikiText-2. Rate is exact within rounding. Harm retains the selected negative-utility sign.}
\begin{tabular}{llrrr}\toprule Target & Router & NLL & Rate & Selected harm\\\midrule
bypass & C48 & 2.6922 & 0.2000 & -0.0281 \\
bypass & C\_RP3 & 2.6878 & 0.2000 & -0.0272 \\
bypass & I\_S35 & 2.6834 & 0.2000 & -0.0269 \\
bypass & I\_ST32 & 2.6833 & 0.2000 & -0.0266 \\
bypass & I\_T32 & 2.6989 & 0.2000 & -0.0243 \\
bypass & J32 & 2.6914 & 0.2000 & -0.0279 \\
bypass & J\_KL & 2.6904 & 0.2000 & -0.0280 \\
bypass & J\_sign & 2.7195 & 0.2000 & -0.0103 \\
bypass & S32 & 2.6807 & 0.2000 & -0.0272 \\
bypass & S32\_T & 2.6803 & 0.2000 & -0.0273 \\
local & C48 & 2.5768 & 0.2000 & -0.0086 \\
local & C\_RP3 & 2.5761 & 0.2000 & -0.0086 \\
local & I\_S35 & 2.5761 & 0.2000 & -0.0084 \\
local & I\_ST32 & 2.5759 & 0.2000 & -0.0088 \\
local & I\_T32 & 2.5775 & 0.2000 & -0.0071 \\
local & J32 & 2.5770 & 0.2000 & -0.0090 \\
local & J\_KL & 2.5767 & 0.2000 & -0.0091 \\
local & J\_sign & 2.5781 & 0.2000 & -0.0060 \\
local & S32 & 2.5766 & 0.2000 & -0.0083 \\
local & S32\_T & 2.5762 & 0.2000 & -0.0087 \\
\bottomrule\end{tabular}\end{table}
\begin{table}[H]\centering\small
\caption{Secondary paired feature/objective contrasts at 20\% quota, $10^{-3}$ nats/token. These are unadjusted 95\% intervals; no additional confirmatory significance claim is made.}
\begin{tabular}{llr}\toprule Target & Contrast & Gain [95\% interval]\\\midrule
bypass & J32--C48 & 0.856 [-0.020, 1.736] \\
bypass & J32--C\_RP3 & -3.557 [-4.874, -2.182] \\
bypass & S32\_T--S32 & 0.426 [-0.347, 1.221] \\
bypass & J32--S32 & -10.676 [-12.569, -8.871] \\
bypass & I\_ST32--I\_S35 & 0.087 [-0.649, 0.838] \\
bypass & J32--J\_sign & 28.125 [26.114, 30.133] \\
bypass & J32--J\_KL & -0.961 [-1.891, 0.014] \\
local & J32--C48 & -0.185 [-0.537, 0.174] \\
local & J32--C\_RP3 & -0.836 [-1.375, -0.328] \\
local & S32\_T--S32 & 0.339 [0.076, 0.617] \\
local & J32--S32 & -0.409 [-0.950, 0.142] \\
local & I\_ST32--I\_S35 & 0.258 [-0.012, 0.517] \\
local & J32--J\_sign & 1.145 [0.468, 1.826] \\
local & J32--J\_KL & -0.215 [-0.446, 0.011] \\
\bottomrule\end{tabular}\end{table}
\subsection{qwen3.5-4b-base / wikitext2\_512}
The sample contains 128 chunks, 46 book/article groups and 63,360 scored predictions. Full-operation NLL is 2.1835.
\begin{table}[H]\centering\small
\caption{Paired reference allocation at 20\% known-horizon quota. NLL/rate/harm are token-weighted and averaged over three fits; intervals for contrasts below instead weight books/articles equally. These estimands coincide for balanced PG-19 chunks but can differ for WikiText-2. Rate is exact within rounding. Harm retains the selected negative-utility sign.}
\begin{tabular}{llrrr}\toprule Target & Router & NLL & Rate & Selected harm\\\midrule
bypass & C48 & 2.3575 & 0.2000 & -0.0279 \\
bypass & C\_RP3 & 2.3513 & 0.2000 & -0.0294 \\
bypass & I\_S35 & 2.3216 & 0.2000 & -0.0309 \\
bypass & I\_ST32 & 2.3221 & 0.2000 & -0.0304 \\
bypass & I\_T32 & 2.3735 & 0.2000 & -0.0233 \\
bypass & J32 & 2.3568 & 0.2000 & -0.0285 \\
bypass & J\_KL & 2.3558 & 0.2000 & -0.0290 \\
bypass & J\_sign & 2.4141 & 0.2000 & -0.0072 \\
bypass & S32 & 2.3209 & 0.2000 & -0.0312 \\
bypass & S32\_T & 2.3200 & 0.2000 & -0.0306 \\
local & C48 & 2.2019 & 0.2000 & -0.0086 \\
local & C\_RP3 & 2.2021 & 0.2000 & -0.0085 \\
local & I\_S35 & 2.1959 & 0.2000 & -0.0104 \\
local & I\_ST32 & 2.1950 & 0.2000 & -0.0101 \\
local & I\_T32 & 2.2006 & 0.2000 & -0.0075 \\
local & J32 & 2.2009 & 0.2000 & -0.0089 \\
local & J\_KL & 2.2015 & 0.2000 & -0.0090 \\
local & J\_sign & 2.2075 & 0.2000 & -0.0073 \\
local & S32 & 2.1945 & 0.2000 & -0.0104 \\
local & S32\_T & 2.1949 & 0.2000 & -0.0103 \\
\bottomrule\end{tabular}\end{table}
\begin{table}[H]\centering\small
\caption{Secondary paired feature/objective contrasts at 20\% quota, $10^{-3}$ nats/token. These are unadjusted 95\% intervals; no additional confirmatory significance claim is made.}
\begin{tabular}{llr}\toprule Target & Contrast & Gain [95\% interval]\\\midrule
bypass & J32--C48 & 1.561 [-0.364, 3.533] \\
bypass & J32--C\_RP3 & -6.080 [-9.989, -2.365] \\
bypass & S32\_T--S32 & 0.773 [-0.844, 2.213] \\
bypass & J32--S32 & -36.719 [-42.281, -31.328] \\
bypass & I\_ST32--I\_S35 & -0.470 [-1.959, 1.045] \\
bypass & J32--J\_sign & 56.417 [52.383, 60.394] \\
bypass & J32--J\_KL & -1.625 [-2.855, -0.425] \\
local & J32--C48 & 1.087 [-0.305, 2.877] \\
local & J32--C\_RP3 & 0.971 [-0.506, 2.430] \\
local & S32\_T--S32 & -0.204 [-1.611, 1.195] \\
local & J32--S32 & -5.451 [-8.027, -2.876] \\
local & I\_ST32--I\_S35 & -0.224 [-1.671, 1.154] \\
local & J32--J\_sign & 6.540 [3.279, 10.570] \\
local & J32--J\_KL & 0.584 [-0.150, 1.343] \\
\bottomrule\end{tabular}\end{table}
\subsection{qwen3.5-4b-base / pg19\_4096}
The sample contains 99 chunks, 99 book/article groups and 403,821 scored predictions. Full-operation NLL is 2.3863.
\begin{table}[H]\centering\small
\caption{Paired reference allocation at 20\% known-horizon quota. NLL/rate/harm are token-weighted and averaged over three fits; intervals for contrasts below instead weight books/articles equally. These estimands coincide for balanced PG-19 chunks but can differ for WikiText-2. Rate is exact within rounding. Harm retains the selected negative-utility sign.}
\begin{tabular}{llrrr}\toprule Target & Router & NLL & Rate & Selected harm\\\midrule
bypass & C48 & 2.5045 & 0.2000 & -0.0261 \\
bypass & C\_RP3 & 2.5009 & 0.2000 & -0.0263 \\
bypass & I\_S35 & 2.4975 & 0.2000 & -0.0258 \\
bypass & I\_ST32 & 2.4952 & 0.2000 & -0.0258 \\
bypass & I\_T32 & 2.5111 & 0.2000 & -0.0228 \\
bypass & J32 & 2.5027 & 0.2000 & -0.0263 \\
bypass & J\_KL & 2.5087 & 0.2000 & -0.0234 \\
bypass & J\_sign & 2.5171 & 0.2000 & -0.0179 \\
bypass & S32 & 2.4935 & 0.2000 & -0.0268 \\
bypass & S32\_T & 2.4965 & 0.2000 & -0.0261 \\
local & C48 & 2.4040 & 0.2000 & -0.0100 \\
local & C\_RP3 & 2.4024 & 0.2000 & -0.0107 \\
local & I\_S35 & 2.4019 & 0.2000 & -0.0109 \\
local & I\_ST32 & 2.4026 & 0.2000 & -0.0106 \\
local & I\_T32 & 2.4048 & 0.2000 & -0.0086 \\
local & J32 & 2.4037 & 0.2000 & -0.0100 \\
local & J\_KL & 2.4034 & 0.2000 & -0.0104 \\
local & J\_sign & 2.4048 & 0.2000 & -0.0086 \\
local & S32 & 2.4036 & 0.2000 & -0.0098 \\
local & S32\_T & 2.4035 & 0.2000 & -0.0099 \\
\bottomrule\end{tabular}\end{table}
\begin{table}[H]\centering\small
\caption{Secondary paired feature/objective contrasts at 20\% quota, $10^{-3}$ nats/token. These are unadjusted 95\% intervals; no additional confirmatory significance claim is made.}
\begin{tabular}{llr}\toprule Target & Contrast & Gain [95\% interval]\\\midrule
bypass & J32--C48 & 1.738 [1.356, 2.110] \\
bypass & J32--C\_RP3 & -1.794 [-2.342, -1.263] \\
bypass & S32\_T--S32 & -2.925 [-3.376, -2.490] \\
bypass & J32--S32 & -9.207 [-10.166, -8.250] \\
bypass & I\_ST32--I\_S35 & 2.306 [1.715, 2.939] \\
bypass & J32--J\_sign & 14.397 [13.473, 15.336] \\
bypass & J32--J\_KL & 5.915 [5.276, 6.579] \\
local & J32--C48 & 0.281 [-0.028, 0.582] \\
local & J32--C\_RP3 & -1.236 [-1.600, -0.864] \\
local & S32\_T--S32 & 0.107 [-0.035, 0.244] \\
local & J32--S32 & -0.093 [-0.486, 0.304] \\
local & I\_ST32--I\_S35 & -0.651 [-0.844, -0.456] \\
local & J32--J\_sign & 1.131 [0.790, 1.488] \\
local & J32--J\_KL & -0.251 [-0.598, 0.097] \\
\bottomrule\end{tabular}\end{table}
\section{Fresh-source retention controls}
\begin{table}[H]\centering\small
\caption{Fresh-source retention performance at length 256, averaged over three fixed repaired checkpoints. Accuracy (\%) and query NLL use the first 128 sequences. The nominal B32 eligibility-aware policy retains all eligible VALUE states and is nonbinding. Binding eligible budgets are secondary.}
\begin{tabular}{llrr}\toprule Length & Retention policy & Accuracy (\%) & NLL\\\midrule
256 & full & 100.00 & 0.0026 \\
256 & recent16 & 11.52 & 2.2821 \\
256 & recent\_extended\_B16 & 10.74 & 2.3735 \\
256 & random\_B16 & 17.25 & 2.6083 \\
256 & error\_B16 & 100.00 & 0.0026 \\
256 & displacement\_B16 & 95.90 & 0.2079 \\
256 & current\_utility\_B16 & 11.52 & 2.2821 \\
256 & eligible\_current\_utility\_B16 & 100.00 & 0.0026 \\
256 & recent\_extended\_B32 & 11.78 & 2.5058 \\
256 & random\_B32 & 23.76 & 2.6959 \\
256 & error\_B32 & 100.00 & 0.0026 \\
256 & displacement\_B32 & 100.00 & 0.0026 \\
256 & current\_utility\_B32 & 12.63 & 2.4243 \\
256 & eligible\_current\_utility\_B32 & 100.00 & 0.0026 \\
256 & recent\_extended\_B64 & 21.81 & 2.5633 \\
256 & random\_B64 & 41.08 & 2.3788 \\
256 & error\_B64 & 100.00 & 0.0026 \\
256 & displacement\_B64 & 100.00 & 0.0026 \\
256 & current\_utility\_B64 & 20.44 & 2.6658 \\
256 & eligible\_current\_utility\_B64 & 100.00 & 0.0026 \\
256 & eligible\_random\_bindingB4 & 42.77 & 2.4462 \\
256 & eligible\_current\_utility\_bindingB4 & 40.17 & 2.4944 \\
256 & eligible\_error\_bindingB4 & 42.58 & 2.4493 \\
256 & eligible\_displacement\_bindingB4 & 39.13 & 2.8176 \\
256 & eligible\_fact\_marker\_bindingB4 & 51.63 & 2.1607 \\
256 & eligible\_random\_bindingB8 & 79.36 & 0.9283 \\
256 & eligible\_current\_utility\_bindingB8 & 76.69 & 1.0417 \\
256 & eligible\_error\_bindingB8 & 82.94 & 0.8099 \\
256 & eligible\_displacement\_bindingB8 & 81.18 & 0.9527 \\
256 & eligible\_fact\_marker\_bindingB8 & 100.00 & 0.0027 \\
256 & eligible\_random\_bindingB16 & 100.00 & 0.0026 \\
256 & eligible\_current\_utility\_bindingB16 & 100.00 & 0.0026 \\
256 & eligible\_error\_bindingB16 & 100.00 & 0.0026 \\
256 & eligible\_displacement\_bindingB16 & 100.00 & 0.0026 \\
256 & eligible\_fact\_marker\_bindingB16 & 100.00 & 0.0026 \\
\bottomrule\end{tabular}\end{table}
\begin{table}[H]\centering\small
\caption{Fresh-source retention performance at length 384, averaged over three fixed repaired checkpoints. Accuracy (\%) and query NLL use the first 128 sequences. The nominal B32 eligibility-aware policy retains all eligible VALUE states and is nonbinding. Binding eligible budgets are secondary.}
\begin{tabular}{llrr}\toprule Length & Retention policy & Accuracy (\%) & NLL\\\midrule
384 & full & 99.09 & 0.1443 \\
384 & recent16 & 12.70 & 2.1798 \\
384 & recent\_extended\_B16 & 10.74 & 2.2812 \\
384 & random\_B16 & 17.25 & 2.3107 \\
384 & error\_B16 & 99.09 & 0.1443 \\
384 & displacement\_B16 & 93.62 & 0.4022 \\
384 & current\_utility\_B16 & 12.76 & 2.1793 \\
384 & eligible\_current\_utility\_B16 & 99.09 & 0.1443 \\
384 & recent\_extended\_B32 & 12.04 & 2.3836 \\
384 & random\_B32 & 19.47 & 2.4393 \\
384 & error\_B32 & 99.09 & 0.1443 \\
384 & displacement\_B32 & 99.09 & 0.1441 \\
384 & current\_utility\_B32 & 13.61 & 2.2966 \\
384 & eligible\_current\_utility\_B32 & 99.09 & 0.1443 \\
384 & recent\_extended\_B64 & 13.15 & 2.5949 \\
384 & random\_B64 & 25.20 & 2.6615 \\
384 & error\_B64 & 99.09 & 0.1443 \\
384 & displacement\_B64 & 99.09 & 0.1443 \\
384 & current\_utility\_B64 & 24.22 & 2.4116 \\
384 & eligible\_current\_utility\_B64 & 99.09 & 0.1443 \\
384 & eligible\_random\_bindingB4 & 37.89 & 2.4993 \\
384 & eligible\_current\_utility\_bindingB4 & 30.99 & 2.5530 \\
384 & eligible\_error\_bindingB4 & 45.96 & 2.3154 \\
384 & eligible\_displacement\_bindingB4 & 40.56 & 2.5035 \\
384 & eligible\_fact\_marker\_bindingB4 & 47.92 & 2.2682 \\
384 & eligible\_random\_bindingB8 & 75.72 & 1.1069 \\
384 & eligible\_current\_utility\_bindingB8 & 66.54 & 1.4080 \\
384 & eligible\_error\_bindingB8 & 84.31 & 0.8030 \\
384 & eligible\_displacement\_bindingB8 & 79.49 & 1.0472 \\
384 & eligible\_fact\_marker\_bindingB8 & 99.22 & 0.1406 \\
384 & eligible\_random\_bindingB16 & 99.09 & 0.1443 \\
384 & eligible\_current\_utility\_bindingB16 & 99.09 & 0.1443 \\
384 & eligible\_error\_bindingB16 & 99.09 & 0.1443 \\
384 & eligible\_displacement\_bindingB16 & 99.09 & 0.1443 \\
384 & eligible\_fact\_marker\_bindingB16 & 99.09 & 0.1443 \\
\bottomrule\end{tabular}\end{table}
\begin{table}[H]\centering\small
\caption{The sixteen prespecified fresh-source contrasts. Accuracy gains are percentage points; NLL reductions are nats/query token. Positive favors the first policy named. Intervals resample 128 independent sequences, retaining paired checkpoint averages. Holm correction applies to all sixteen endpoints.}
\begin{tabular}{llrr}\toprule Contrast & Endpoint & Gain [95\% interval] & Adjusted $P$\\\midrule
L256 random & Accuracy & 11.133 [7.552, 14.844] & 0.0016 \\
L256 random & NLL & -0.272 [-0.410, -0.139] & 0.0016 \\
L256 error & Accuracy & 87.370 [85.156, 89.518] & 0.0016 \\
L256 error & NLL & 2.422 [2.368, 2.479] & 0.0016 \\
L256 displacement & Accuracy & 87.370 [85.156, 89.518] & 0.0016 \\
L256 displacement & NLL & 2.422 [2.368, 2.479] & 0.0016 \\
L256 eligible current utility & Accuracy & 87.370 [85.156, 89.518] & 0.0016 \\
L256 eligible current utility & NLL & 2.422 [2.368, 2.479] & 0.0016 \\
L384 random & Accuracy & 5.859 [2.865, 8.919] & 0.0016 \\
L384 random & NLL & -0.143 [-0.239, -0.050] & 0.0030 \\
L384 error & Accuracy & 85.482 [83.333, 87.565] & 0.0016 \\
L384 error & NLL & 2.152 [2.102, 2.206] & 0.0016 \\
L384 displacement & Accuracy & 85.482 [83.333, 87.565] & 0.0016 \\
L384 displacement & NLL & 2.152 [2.102, 2.206] & 0.0016 \\
L384 eligible current utility & Accuracy & 85.482 [83.333, 87.565] & 0.0016 \\
L384 eligible current utility & NLL & 2.152 [2.102, 2.206] & 0.0016 \\
\bottomrule\end{tabular}\end{table}
\section{Historical causal-quota context}
The earlier byte models and GPT-2 motivated the modern follow-up. Their stronger displacement/position controls materially reduced the apparent trajectory gain. Natural-language causal-quota intervals for the strongest control included zero in all three trained configurations. The GPT-2 contrast was positive for one fixed checkpoint. These exploratory results are retained in their own units and are not added to the modern primary family.
\begin{table}[H]
\centering\small
\caption{Prefix-only known-horizon quota policy at matched optional-branch call counts, $10^{-3}$ nats per unit. Natural scored horizons round to 48/239 calls (20.08\%); GPT-2 uses 99/495 (20.00\%). Brackets condition on the fitted models and resample articles; this is an exploratory deadline heuristic, not a measured-cost or optimal-policy guarantee.}
\label{table:quota}
\begin{tabular}{llcc}
\toprule
Model & Unit & Joint $-$ uncertainty & Both control $D$ and position \\
\midrule
3.48M & byte & 0.528 [0.271, 0.782] & 0.049 [-0.217, 0.310] \\
20.08M & byte & 0.537 [-0.165, 1.229] & 0.234 [-0.244, 0.789] \\
3.48M, balanced & byte & 0.161 [0.057, 0.265] & 0.097 [-0.038, 0.227] \\
GPT-2 124M & BPE token & 20.574 [16.974, 24.215] & 7.987 [4.944, 10.846] \\
\bottomrule
\end{tabular}
\end{table}

\section{Historical training and exploratory evaluation}
\label{sec:experiments}

\paragraph{Natural-language training and independent splits.}
We downloaded WikiText-2 raw from the official Salesforce repository at revision \texttt{b08601e04326c79dfdd32d625aee71d232d685c3} \citep{merity2017pointer}. Level-one article headings defined source groups. The official training split was divided in article order into 549 backbone-training articles (9,669,521 UTF-8 bytes) and 61 router-supervision articles (1,244,408 bytes). Official validation and test supplied 60 and 61 articles, respectively. Exact article hashes did not overlap across these four roles. We used a fixed 256-byte vocabulary with no learned tokenizer. Article-local, nonoverlapping 257-byte segments supplied 256 next-byte predictions; a seeded sample of 512 segments per router-supervision, validation, and test role was evaluated. Removing the first 17 positions left 122,368 test predictions from 59 articles. Thus reported bits per byte (BPB) concern this finite-context evaluation, rather than a published word-level WikiText perplexity protocol.

The primary hybrid had four local causal Transformer blocks, width 256, four attention heads, window 16 including the current position, and one output-only global attention correction. Its parameter count was 3,483,648. Linear layers used Gaussian initialization with standard deviation 0.02, GELU feed-forward expansion four, LayerNorm, a tied byte head, and fixed sinusoidal positions. The original position amplitude was one. Three seeds (11, 23, 37) each trained for 6,000 updates, batch 32, context 256, or 49.152 million byte predictions. AdamW used peak learning rate $5\times10^{-4}$, 100 warmup updates, cosine decay to 10\% of peak, weight decay 0.1, betas $(0.9,0.95)$, and gradient clipping one. The hybrid minimized the equally weighted mean of fast and corrected losses. No routed outputs changed later fast states.

A larger probe used width 512 and six local blocks with the other settings unchanged. A small-scale positional-amplitude probe used 0.02 for both hybrid and dense variants, balancing position and token initialization scales. Each architecture also had a dense same-shape reference with every backbone attention mask causal and unrestricted, trained on the same tokens and updates using corrected loss alone. These references were not an extensively tuned dense baseline, and the different objectives preclude an equal-training-FLOP claim. The local receptive fields were 61 input positions for four blocks and 91 for six; changing width and depth therefore did not isolate parameter count alone. The larger scale, balanced-position probe, and uncertainty-plus-displacement control were exploratory additions after the first small-model outcomes, documented as such rather than presented as preregistered confirmation.

\paragraph{Frozen public-model intervention.}
We downloaded GPT-2 (124,439,808 parameters) at repository revision \texttt{607a30d783dfa663caf39e06633721c8d4cfcd7e}, verifying the official weight SHA-256 \citep{radford2019language}. The first eleven dense blocks were shared. In the final block we enabled or bypassed its attention residual; both branches retained that block's MLP, final normalization, and tied head. Trajectory features used the bypass output before final normalization. This diagnostic tests a same-state operation intervention in a trained public model; eleven dense blocks still ran at every position, so it was not a cheap CURVE-LM architecture. A custom forward pass matched the official Transformers implementation on CPU, with maximum absolute validation-logit difference $1.68\times10^{-4}$ and identical validation NLL. Future-token mutation changed preceding paired states and features by exactly zero.

The same article partitions were independently tokenized with the pinned GPT-2 BPE tokenizer. Each role used 128 nonoverlapping 513-token segments and 512 predictions. After the common warmup, test contained 63,360 BPE predictions from 42 articles. Collection used FP32, batch two. GPT-2 NLL is in nats per BPE token; its causal perplexity is not comparable numerically to the byte models' BPB. We did not fine-tune this checkpoint or claim independent backbone-training replicates for it.

\paragraph{Matched router controls and calibration.}
The learned subsets were uncertainty ($H$, probability margin), displacement ($D$), uncertainty plus displacement, trajectory ($E,C,D,A$), and joint trajectory plus uncertainty. Position-only and position-augmented variants tested a context-opportunity confound. A fixed orthogonal 32-dimensional projection of the current normalized state, followed by the same width-32 MLP, was a stronger small-model state control. It had 1,089 learned parameters versus 257 for the joint router; it was not a reproduction of Mixture-of-Depths. Projection seed 127 was fixed independently of outcomes. Raw entropy, error, displacement, random scores, and a target-informed oracle were also reported.

For natural language and GPT-2, each learned router used 1,000 AdamW updates with learning rate 0.002, weight decay 0.01, batch 1,024, mean-squared utility loss, and initialization seed 901. Training-feature means/standard deviations and target scaling were fit only on router-supervision data. Validation MSE selected a checkpoint every 100 updates. The same optimization recipe applied to all subsets within a study. Budgets were 5, 10, 20, 30, and 50\%. Retrospective ranking selected exactly the requested fraction of the scored test bank. Causal thresholds were validation-score quantiles frozen before test, with actual rates reported. Exact-quota diagnostics may include predicted negative values; they do not implement Proposition~\ref{supp:budget}'s positive-clipped upper-bound optimizer. No test target entered a learned decision.

\paragraph{Associative recall and retention.}
We separately trained a width-128, two-block model on generated length-256 sequences with a 40-token vocabulary. Each sequence had eight early FACT, key and value declarations, four delayed QUERY events with associated values, four NOVEL events with independent values, and uniform random filler. Values, keys, order, positions, and filler varied across independent seeds. Query and novel targets shared the eight-value marginal distribution, while markers deliberately indicated retrieval versus new sampling. The minimum fact-to-query delay exceeded the local receptive field. This engineered task isolates retrieval; success does not establish a general semantic context-switch mechanism.

Three backbone seeds each trained on 1,000 fresh batches of size 64, with AdamW peak learning rate 0.0007 and the same warmup/decay/clipping conventions. Every value-category target, including initial and novel facts, received weight four; evaluation was unweighted. Independent generator seeds supplied 384 router-training, 256 calibration, 512 test sequences, plus 256 length-384 shift sequences. Routers used 600 updates under a common within-study recipe. They received no event label or target value.

Retention diagnostics attended to the most recent 16 keys plus up to $B\in\{16,32,64\}$ historical keys, using recent, prefix-uniform stride, random priority, error, displacement, predicted current utility, or normalized utility-plus-error priorities. Masks at position $t$ used only scores from its prefix and selected equal nominal retention-slot counts. Fast states were materialized beforehand, so this tested selection quality rather than bounded-storage deployment. It did not reproduce H$_2$O or SnapKV. Retention was evaluated separately on 64 test sequences and 64 shifted sequences per seed.

\paragraph{Endpoints, uncertainty, and execution.}
The initial primary endpoint was signed captured utility for joint minus learned uncertainty at an exact retrospective 20\% query budget. Positive captured utility was a secondary metric because it omits selected harms. Correlation, AUROC for positive utility, window $k\in\{3,4,8,16\}$, and stronger controls characterized the signal. Paired 95\% percentile intervals used 2,000 whole-article bootstrap resamples for natural language/GPT-2 and whole-sequence resamples for generated data. These conditional intervals describe evaluation-unit variability for fixed fitted models; three-seed mean and sample standard deviation describe a separate training-variation axis. Added contrasts and sensitivity analyses were exploratory, with no multiple-benchmark significance claim.

All GPU experiments used an RTX 5080 (16,303 MiB), PyTorch 2.11.0 with CUDA 12.8, and a Ryzen 9 9950X3D host. Small-model training used BF16 autocast with FP32 losses and optimizer states. This Windows build used efficient SDPA; a forced FlashAttention check found no available kernel. We timed actual selected-query correction, including feature extraction, the pre-gate vocabulary head, router, synchronization, gathering, all context projections, and required output heads. Optimized always-global inference used a single final vocabulary head. Prefill/scoring trials covered contexts 256, 1,024, and 2,048 at stated batches; longer contexts were timing probes, not trained quality benchmarks. Incremental local caches and an accumulating global KV store timed a fixed ground-truth continuation, not free-generation task quality. Warmup was discarded, GPU execution synchronized, and raw wall-clock trials retained. CPU FP32 invariants and GPU BF16 NLL/logit checks verified cache and branch equivalence before interpreting timing.

\paragraph{Exploratory causal quotas.}
After the initial outcomes were observed, we recorded a fixed follow-up rule
before computing its results. For each saved scored segment of known length
$N$ and nominal budget $\rho$, the rule allocated $k=\operatorname{round}(\rho N)$
calls, using Python's rounding convention. It reused each router's existing
validation-calibrated threshold $\tau$, without fitting or recalibration.
Let $q_t$ be remaining calls and $r_t$ remaining positions, including the current
position, using zero-based scored positions with $q_0=k$ and $r_0=N$. The decision rejected when $q_t=0$, activated when $q_t=r_t>0$,
and otherwise activated exactly when the current score exceeded $\tau$.
The state then updated as $q_{t+1}=q_t-g_t$ and $r_{t+1}=r_t-1$. The horizon was the scored suffix
after warmup, and quota state reset for each segment.

The invariant $0\le q_t\le r_t$ holds initially and after every decision.
At the deadline $r_N=0$, it implies $q_N=0$, hence exactly $k$ activations.
Each decision depends only on the current score, a fixed threshold, and state
determined by earlier decisions; changing future scores therefore cannot change
an earlier gate. Tests checked integer counts, ties, zero/full quotas, and prefix
invariance under future-score mutation. All budgets and deployable saved routers
were evaluated, with target-aware oracles excluded. Paired routed losses and
synthetic query accuracy were computed from saved branch outcomes; 20\% contrasts
used 2,000 article/sequence bootstrap resamples with seed 217.

This is a teacher-forced deadline heuristic on declared fixed horizons, not an
optimal selector or an unknown-EOS serving policy. Rounding slightly changes
nominal rates, and deadline forcing can select low-utility positions. Equal
invocation counts do not establish equal latency or equal context-dependent
attention cost. The analysis is exploratory and leaves earlier results unchanged.

\paragraph{Synthetic repair disclosure.}
The original three-seed probe used 1,000 updates, position scale one, and value weight four. A subsequent three-seed probe used 6,000 updates, position scale 0.02, and value weight 16, but still failed retrieval. The final three-seed support configuration kept the latter schedule and added two training-only untied 40-class heads (10,320 parameters) reading the correction-normalized fast state. Current and previous observed-token CE had weights 0.5 each; the previous loss excluded the first position. The correction context was also restricted to observed prefix value tokens 16--23. Both modifications were made jointly before observing support-run outcomes; no causal attribution to the auxiliary objective alone is made. Auxiliary heads were discarded at inference. These 466,688-parameter support models used the same generator and held-out splits, not future targets or generator answer mappings as inputs. Every retention mask selected the same number of nominal slots from materialized states; the correction could use only eligible values among them. The repairs were selected after inspecting outcomes on the same fixed test generator sources, so the support results are exploratory rather than untouched confirmatory evaluation. A CPU linear recoverability diagnostic already found 91.4\% previous-key recovery in the longer failed model, so information absence alone does not explain its failure.

For fixed-seed mean intervals, we averaged paired per-position differences across the three models, then resampled the shared article IDs. This interval is conditional on those fitted models and is not a confidence interval over future model-training runs. All added contrasts are exploratory. The raw archives retain sign-AUROC, Pearson/Spearman, captured positive utility, selected harms, all budgets, and failed configurations; no universal improvement on every predictive metric is claimed.

\section{Historical configurations, seed variation and failures}

\begin{table}[H]
\centering\small
\caption{Individual-seed primary contrasts at exact 20\% budgets, $10^{-3}$ nats/byte. These remain visible alongside the fixed-seed mean intervals.}
\label{table:seeds}
\begin{tabular}{llrr}
\toprule
Configuration & Seed & Paired gain & 95\% article interval \\
\midrule
3.48M & 11 & 0.738 & [0.301, 1.211] \\
3.48M & 23 & 0.518 & [0.197, 0.859] \\
3.48M & 37 & 0.418 & [0.023, 0.846] \\
20.08M & 11 & 1.859 & [0.956, 2.758] \\
20.08M & 23 & 1.562 & [0.527, 2.644] \\
20.08M & 37 & 0.830 & [-0.187, 1.977] \\
Balanced 3.48M & 11 & 0.232 & [0.091, 0.381] \\
Balanced 3.48M & 23 & 0.299 & [0.101, 0.517] \\
Balanced 3.48M & 37 & -0.015 & [-0.215, 0.155] \\
\bottomrule
\end{tabular}
\end{table}

\begin{table}[H]
\centering\small
\caption{Finite-context test BPB, mean $\pm$ sample SD across three seeds. Threshold joint uses a frozen validation threshold targeting 20\%, with potentially different actual test rates; it is not the fixed-horizon quota policy. Dense references have the same parameter shapes and token/update exposure, but a different training objective and no extensive tuning. Amplitude sensitivity demonstrates why these are not evidence of superiority over an optimized dense model.}
\label{table:baselines}
\begin{tabular}{lrrrr}
\toprule
Configuration & Fast & Full correction & Threshold joint & Dense reference \\
\midrule
Original 3.48M & 1.997 $\pm$ 0.009 & 1.976 $\pm$ 0.008 & 1.981 $\pm$ 0.008 & 2.201 $\pm$ 0.051 \\
20.08M & 1.877 $\pm$ 0.009 & 1.834 $\pm$ 0.010 & 1.854 $\pm$ 0.009 & 1.949 $\pm$ 0.006 \\
Balanced 3.48M & 1.921 $\pm$ 0.001 & 1.910 $\pm$ 0.002 & 1.912 $\pm$ 0.002 & 1.911 $\pm$ 0.004 \\
\bottomrule
\end{tabular}
\end{table}

\begin{table}[H]
\centering\small
\caption{Trajectory-window sensitivity: signed utility captured by the joint router at an exact retrospective 20\% query budget, $10^{-3}$ nats/byte, mean $\pm$ sample SD over three seeds. Each window used the same scored suffix and router recipe; $k=4$ was the original setting. The values are reductions from the fast branch, not incremental gains over uncertainty. This exploratory diagnostic selects scores retrospectively and does not identify a universally optimal window.}
\label{table:windows}
\begin{tabular}{lrrrr}
\toprule
Configuration & $k=3$ & $k=4$ & $k=8$ & $k=16$ \\
\midrule
Original 3.48M & 10.891 $\pm$ 0.323 & 10.891 $\pm$ 0.061 & 10.662 $\pm$ 0.199 & 10.843 $\pm$ 0.367 \\
20.08M & 16.445 $\pm$ 0.643 & 15.962 $\pm$ 0.418 & 16.299 $\pm$ 0.149 & 16.696 $\pm$ 0.388 \\
Balanced 3.48M & 6.226 $\pm$ 0.469 & 6.262 $\pm$ 0.441 & 6.136 $\pm$ 0.528 & 6.239 $\pm$ 0.530 \\
\bottomrule
\end{tabular}
\end{table}

\begin{table}[H]
\centering\small
\caption{All synthetic training attempts, three seeds per configuration. Unsuccessful correction branches are retained and not used to infer retention quality. The support configuration changed both supervision and context eligibility, so it is not a one-factor causal ablation.}
\label{table:synthetic_attempts}
\begin{tabular}{llrr}
\toprule
Configuration & Length & Fast query accuracy \% & Full query accuracy \% \\
\midrule
1k pilot & 256 & 12.16 $\pm$ 1.02 & 11.95 $\pm$ 1.36 \\
1k pilot & 384 & 8.79 $\pm$ 0.74 & 8.20 $\pm$ 0.74 \\
6k revision & 256 & 13.51 $\pm$ 0.45 & 13.46 $\pm$ 0.57 \\
6k revision & 384 & 12.73 $\pm$ 1.80 & 12.70 $\pm$ 1.76 \\
Support configuration & 256 & 12.74 $\pm$ 0.59 & 100.00 $\pm$ 0.00 \\
Support configuration & 384 & 11.85 $\pm$ 0.31 & 99.25 $\pm$ 0.46 \\
\bottomrule
\end{tabular}
\end{table}

\begin{figure}[H]
\centering
\includegraphics[width=\linewidth]{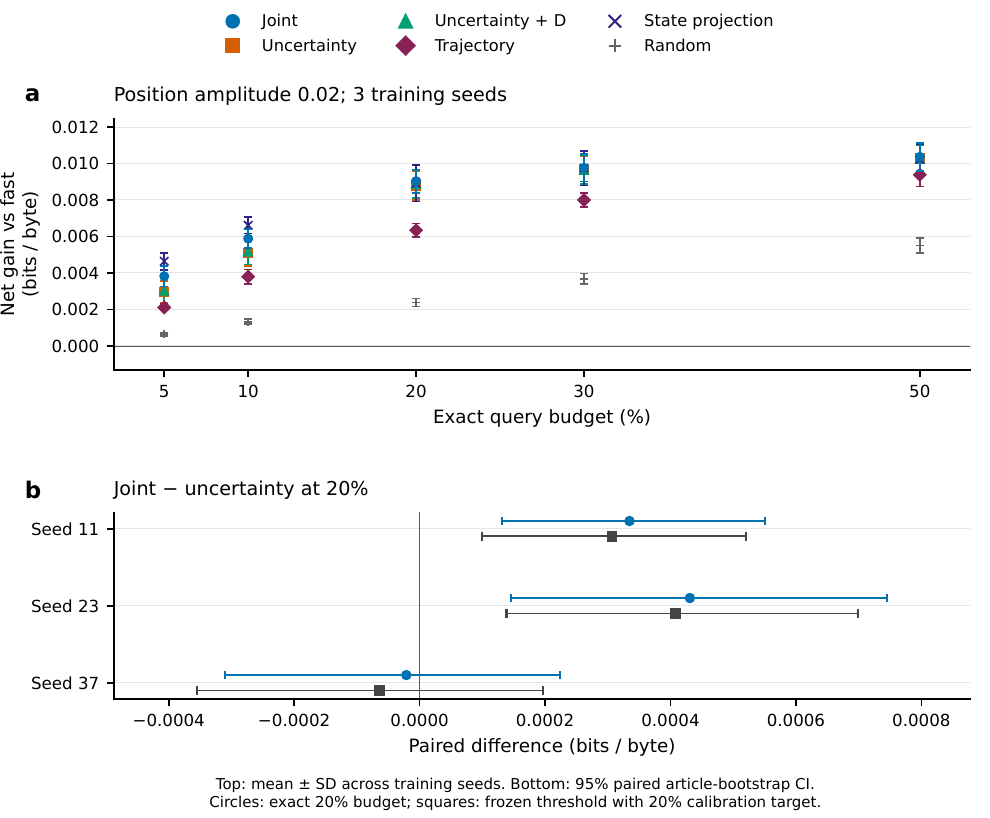}
\caption{Outcome-triggered balanced-position sensitivity probe. All three seeds and the same gate/threshold distinction are shown. The reduced mean incremental gain is a material qualification of the original configuration, not an omitted failure.}
\label{fig:balanced}
\end{figure}

\begin{figure}[H]
\centering
\includegraphics[width=\linewidth]{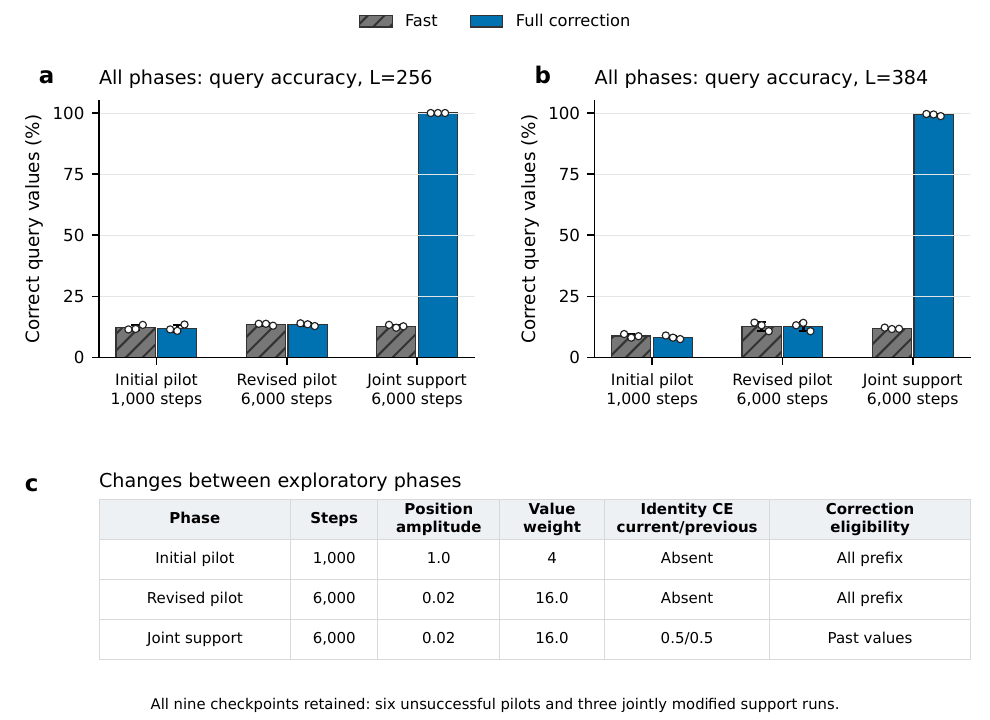}
\caption{All nine synthetic checkpoints and their recorded training configurations. Initial and longer attempts failed associative recall. The support configuration jointly changed causal reconstruction supervision and value eligibility and is explicitly exploratory. No unsuccessful seed or configuration is discarded.}
\label{fig:synthetic_attempts}
\end{figure}

\begin{figure}[H]
\centering
\includegraphics[width=\linewidth]{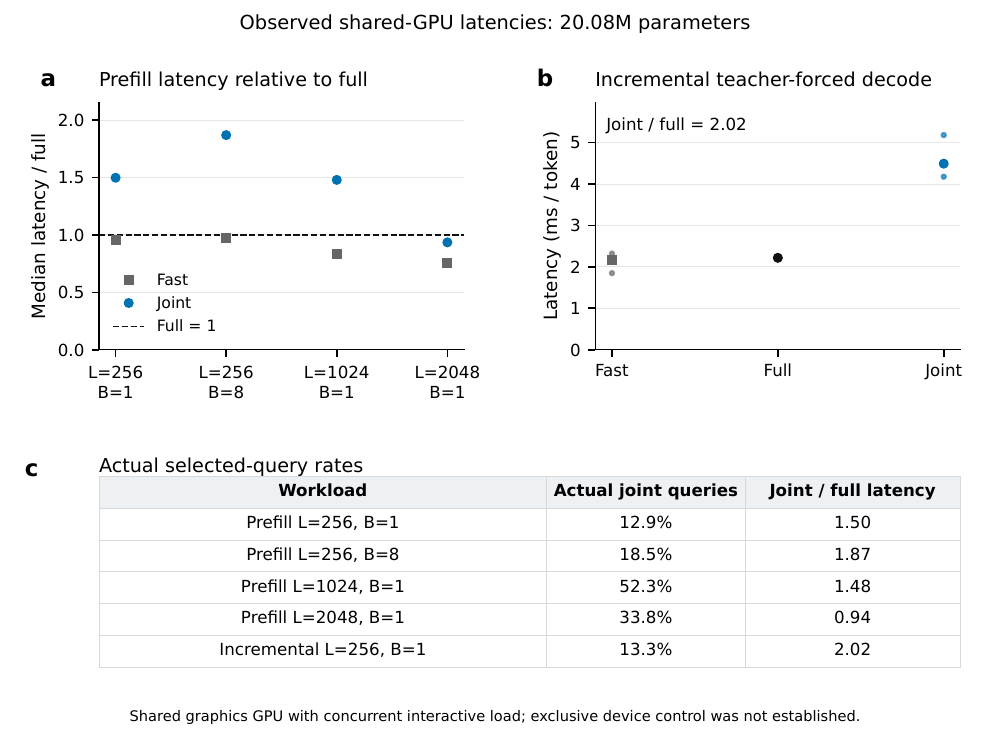}
\caption{Additional 20.08M implementation timing under shared graphics-device load. These observations retain unfavorable ratios and are not isolated-device or quality-matched deployment benchmarks.}
\label{fig:scale_timing}
\end{figure}

\section{Executed quality and systems measurements}
\subsection{Qwen3.5: cached prompt1024 forced64}
Recorded status: complete. The native automatic backend and MATH numerical reference retain separate executed losses.
\begin{table}[H]\centering\small
\caption{Actual executed policies, fixed fitting seed 901. NLL uses the same 30 timing-sample blocks; it is not the complete-roster sequence-quality endpoint. Median and p95 are whole-region milliseconds. Ratio is native automatic full median divided by policy median; values above one favor the policy. Block intervals quantify repeated timing under recorded shared-desktop conditions, not unseen devices or free-generation performance.}
\begin{tabular}{lrrrr}\toprule Policy & Median & p95 & Executed NLL & Native speed ratio [95\%]\\\midrule
native\_auto\_full & 1478.9 & 1515.9 & 2.5589 & 1.000 [1.000, 1.000] \\
matched\_math\_full & 1462.0 & 1524.8 & 2.5589 & 1.012 [0.987, 1.021] \\
true\_bypass & 1475.5 & 1563.2 & 2.7064 & 1.002 [0.976, 1.017] \\
random\_quota & 1473.8 & 1572.2 & 2.6808 & 1.003 [0.986, 1.013] \\
I\_T32 & 1514.2 & 1571.8 & 2.6759 & 0.977 [0.951, 0.987] \\
I\_S35 & 1516.3 & 1556.8 & 2.6648 & 0.975 [0.958, 0.989] \\
I\_ST32 & 1528.3 & 1574.1 & 2.6633 & 0.968 [0.945, 0.985] \\
C48 & 1649.1 & 1754.8 & 2.6677 & 0.897 [0.879, 0.912] \\
J32 & 1654.8 & 1801.4 & 2.6656 & 0.894 [0.872, 0.909] \\
S32\_T & 1671.6 & 1716.2 & 2.6623 & 0.885 [0.872, 0.902] \\
\bottomrule\end{tabular}\end{table}
\subsection{Qwen3.5: cached prompt512 forced64}
Recorded status: complete. The native automatic backend and MATH numerical reference retain separate executed losses.
\begin{table}[H]\centering\small
\caption{Actual executed policies, fixed fitting seed 901. NLL uses the same 30 timing-sample blocks; it is not the complete-roster sequence-quality endpoint. Median and p95 are whole-region milliseconds. Ratio is native automatic full median divided by policy median; values above one favor the policy. Block intervals quantify repeated timing under recorded shared-desktop conditions, not unseen devices or free-generation performance.}
\begin{tabular}{lrrrr}\toprule Policy & Median & p95 & Executed NLL & Native speed ratio [95\%]\\\midrule
native\_auto\_full & 1475.9 & 1547.0 & 2.5242 & 1.000 [1.000, 1.000] \\
matched\_math\_full & 1454.5 & 1503.2 & 2.5242 & 1.015 [0.999, 1.025] \\
true\_bypass & 1470.4 & 1547.1 & 2.6980 & 1.004 [0.988, 1.024] \\
random\_quota & 1452.7 & 1536.3 & 2.6592 & 1.016 [1.003, 1.024] \\
I\_T32 & 1497.9 & 1556.9 & 2.6605 & 0.985 [0.967, 0.995] \\
I\_S35 & 1503.3 & 1592.0 & 2.6471 & 0.982 [0.966, 0.997] \\
I\_ST32 & 1522.9 & 1559.6 & 2.6480 & 0.969 [0.956, 0.983] \\
C48 & 1613.9 & 1669.6 & 2.6505 & 0.914 [0.898, 0.923] \\
J32 & 1625.4 & 1672.1 & 2.6446 & 0.908 [0.890, 0.918] \\
S32\_T & 1628.4 & 1741.0 & 2.6421 & 0.906 [0.884, 0.915] \\
\bottomrule\end{tabular}\end{table}
\subsection{Qwen3.5: scoring 1024}
Recorded status: complete. The native automatic backend and MATH numerical reference retain separate executed losses.
\begin{table}[H]\centering\small
\caption{Actual executed policies, fixed fitting seed 901. Scoring NLL uses the complete locked roster; latency uses the first 30 rows. Median and p95 are whole-region milliseconds. Ratio is native automatic full median divided by policy median; values above one favor the policy. Block intervals quantify repeated timing under recorded shared-desktop conditions, not unseen devices or free-generation performance.}
\begin{tabular}{lrrrr}\toprule Policy & Median & p95 & Executed NLL & Native speed ratio [95\%]\\\midrule
native\_auto\_full & 192.4 & 210.5 & 2.4973 & 1.000 [1.000, 1.000] \\
matched\_math\_full & 193.2 & 207.0 & 2.4973 & 0.996 [0.988, 1.006] \\
true\_bypass & 190.4 & 206.9 & 2.6602 & 1.010 [1.001, 1.019] \\
random\_quota & 191.0 & 217.2 & 2.6288 & 1.007 [0.999, 1.014] \\
I\_T32 & 193.9 & 209.2 & 2.6248 & 0.992 [0.989, 1.000] \\
I\_S35 & 191.2 & 211.2 & 2.6064 & 1.006 [0.994, 1.012] \\
I\_ST32 & 193.8 & 209.5 & 2.6067 & 0.993 [0.987, 1.006] \\
C48 & 230.4 & 258.5 & 2.6154 & 0.835 [0.832, 0.841] \\
J32 & 229.3 & 248.4 & 2.6139 & 0.839 [0.833, 0.844] \\
S32\_T & 229.5 & 251.0 & 2.6037 & 0.838 [0.827, 0.843] \\
\bottomrule\end{tabular}\end{table}
\subsection{Qwen3.5: scoring 4096}
Recorded status: complete. The native automatic backend and MATH numerical reference retain separate executed losses.
\begin{table}[H]\centering\small
\caption{Actual executed policies, fixed fitting seed 901. Scoring NLL uses the complete locked roster; latency uses the first 30 rows. Median and p95 are whole-region milliseconds. Ratio is native automatic full median divided by policy median; values above one favor the policy. Block intervals quantify repeated timing under recorded shared-desktop conditions, not unseen devices or free-generation performance.}
\begin{tabular}{lrrrr}\toprule Policy & Median & p95 & Executed NLL & Native speed ratio [95\%]\\\midrule
native\_auto\_full & 975.8 & 1066.9 & 2.3863 & 1.000 [1.000, 1.000] \\
matched\_math\_full & 996.2 & 1056.3 & 2.3863 & 0.980 [0.953, 1.017] \\
true\_bypass & 938.2 & 1016.7 & 2.5443 & 1.040 [1.035, 1.062] \\
random\_quota & 949.1 & 1038.9 & 2.5127 & 1.028 [0.988, 1.049] \\
I\_T32 & 949.8 & 1074.7 & 2.5113 & 1.027 [1.020, 1.048] \\
I\_S35 & 947.9 & 1002.1 & 2.4965 & 1.029 [1.024, 1.050] \\
I\_ST32 & 953.5 & 1018.2 & 2.4965 & 1.023 [0.999, 1.045] \\
C48 & 1129.0 & 1248.4 & 2.5007 & 0.864 [0.847, 0.894] \\
J32 & 1101.5 & 1150.5 & 2.4982 & 0.886 [0.879, 0.905] \\
S32\_T & 1103.8 & 1205.0 & 2.4983 & 0.884 [0.861, 0.904] \\
\bottomrule\end{tabular}\end{table}
\subsection{Qwen3.5: scoring 512}
Recorded status: complete. The native automatic backend and MATH numerical reference retain separate executed losses.
\begin{table}[H]\centering\small
\caption{Actual executed policies, fixed fitting seed 901. Scoring NLL uses the complete locked roster; latency uses the first 30 rows. Median and p95 are whole-region milliseconds. Ratio is native automatic full median divided by policy median; values above one favor the policy. Block intervals quantify repeated timing under recorded shared-desktop conditions, not unseen devices or free-generation performance.}
\begin{tabular}{lrrrr}\toprule Policy & Median & p95 & Executed NLL & Native speed ratio [95\%]\\\midrule
native\_auto\_full & 96.2 & 105.1 & 2.5703 & 1.000 [1.000, 1.000] \\
matched\_math\_full & 96.1 & 105.5 & 2.5703 & 1.001 [0.989, 1.017] \\
true\_bypass & 95.1 & 102.6 & 2.7364 & 1.012 [0.998, 1.028] \\
random\_quota & 95.9 & 105.2 & 2.7024 & 1.003 [0.986, 1.018] \\
I\_T32 & 96.3 & 102.7 & 2.6991 & 1.000 [0.994, 1.014] \\
I\_S35 & 96.4 & 104.9 & 2.6840 & 0.998 [0.992, 1.014] \\
I\_ST32 & 96.8 & 106.1 & 2.6829 & 0.994 [0.985, 1.009] \\
C48 & 115.8 & 125.5 & 2.6922 & 0.831 [0.822, 0.844] \\
J32 & 116.3 & 126.8 & 2.6914 & 0.827 [0.815, 0.837] \\
S32\_T & 115.7 & 123.8 & 2.6808 & 0.831 [0.821, 0.844] \\
\bottomrule\end{tabular}\end{table}
\subsection{SmolLM3: cached prompt1024 forced64}
Recorded status: complete. The native automatic backend and MATH numerical reference retain separate executed losses.
\begin{table}[H]\centering\small
\caption{Actual executed policies, fixed fitting seed 901. NLL uses the same 30 timing-sample blocks; it is not the complete-roster sequence-quality endpoint. Median and p95 are whole-region milliseconds. Ratio is native automatic full median divided by policy median; values above one favor the policy. Block intervals quantify repeated timing under recorded shared-desktop conditions, not unseen devices or free-generation performance.}
\begin{tabular}{lrrrr}\toprule Policy & Median & p95 & Executed NLL & Native speed ratio [95\%]\\\midrule
native\_auto\_full & 1060.0 & 1100.4 & 2.5742 & 1.000 [1.000, 1.000] \\
matched\_math\_full & 1060.8 & 1114.2 & 2.5742 & 0.999 [0.983, 1.014] \\
true\_bypass & 1072.2 & 1153.5 & 2.6103 & 0.989 [0.970, 1.002] \\
random\_quota & 1077.5 & 1159.3 & 2.6043 & 0.984 [0.969, 0.999] \\
I\_T32 & 1114.6 & 1156.2 & 2.6080 & 0.951 [0.934, 0.965] \\
I\_S35 & 1110.5 & 1147.4 & 2.5893 & 0.955 [0.937, 0.974] \\
I\_ST32 & 1129.7 & 1179.0 & 2.5842 & 0.938 [0.928, 0.956] \\
C48 & 1176.1 & 1263.2 & 2.6080 & 0.901 [0.888, 0.915] \\
J32 & 1180.4 & 1233.9 & 2.6081 & 0.898 [0.883, 0.912] \\
S32\_T & 1185.5 & 1271.6 & 2.5983 & 0.894 [0.882, 0.911] \\
\bottomrule\end{tabular}\end{table}
\subsection{SmolLM3: cached prompt512 forced64}
Recorded status: complete. The native automatic backend and MATH numerical reference retain separate executed losses.
\begin{table}[H]\centering\small
\caption{Actual executed policies, fixed fitting seed 901. NLL uses the same 30 timing-sample blocks; it is not the complete-roster sequence-quality endpoint. Median and p95 are whole-region milliseconds. Ratio is native automatic full median divided by policy median; values above one favor the policy. Block intervals quantify repeated timing under recorded shared-desktop conditions, not unseen devices or free-generation performance.}
\begin{tabular}{lrrrr}\toprule Policy & Median & p95 & Executed NLL & Native speed ratio [95\%]\\\midrule
native\_auto\_full & 1049.2 & 1136.4 & 2.4609 & 1.000 [1.000, 1.000] \\
matched\_math\_full & 1047.3 & 1091.0 & 2.4609 & 1.002 [0.982, 1.029] \\
true\_bypass & 1066.9 & 1113.6 & 2.4961 & 0.983 [0.975, 1.016] \\
random\_quota & 1062.3 & 1114.9 & 2.4878 & 0.988 [0.973, 1.016] \\
I\_T32 & 1105.1 & 1153.0 & 2.4884 & 0.949 [0.941, 0.977] \\
I\_S35 & 1089.9 & 1138.3 & 2.4675 & 0.963 [0.956, 0.990] \\
I\_ST32 & 1104.5 & 1137.4 & 2.4682 & 0.950 [0.938, 0.975] \\
C48 & 1153.0 & 1211.1 & 2.4860 & 0.910 [0.904, 0.936] \\
J32 & 1174.5 & 1206.3 & 2.4843 & 0.893 [0.874, 0.923] \\
S32\_T & 1172.4 & 1207.3 & 2.4700 & 0.895 [0.877, 0.921] \\
\bottomrule\end{tabular}\end{table}
\subsection{SmolLM3: scoring 1024}
Recorded status: partial\_failures\_retained. The native automatic backend and MATH numerical reference retain separate executed losses.
Failed arms remain excluded from timing/quality interpretation: I\_S35.
\begin{table}[H]\centering\small
\caption{Actual executed policies, fixed fitting seed 901. Scoring NLL uses the complete locked roster; latency uses the first 30 rows. Median and p95 are whole-region milliseconds. Ratio is native automatic full median divided by policy median; values above one favor the policy. Block intervals quantify repeated timing under recorded shared-desktop conditions, not unseen devices or free-generation performance.}
\begin{tabular}{lrrrr}\toprule Policy & Median & p95 & Executed NLL & Native speed ratio [95\%]\\\midrule
native\_auto\_full & 134.6 & 146.2 & 2.4734 & 1.000 [1.000, 1.000] \\
matched\_math\_full & 134.7 & 142.7 & 2.4734 & 0.999 [0.989, 1.016] \\
true\_bypass & 132.9 & 145.2 & 2.5041 & 1.013 [1.006, 1.028] \\
random\_quota & 134.2 & 141.2 & 2.4978 & 1.004 [1.000, 1.019] \\
I\_T32 & 135.5 & 149.0 & 2.4947 & 0.993 [0.978, 1.010] \\
I\_ST32 & 136.0 & 148.5 & 2.4808 & 0.990 [0.978, 1.011] \\
C48 & 152.5 & 165.1 & 2.4926 & 0.883 [0.874, 0.897] \\
J32 & 152.1 & 157.7 & 2.4920 & 0.885 [0.879, 0.899] \\
S32\_T & 152.1 & 164.9 & 2.4805 & 0.885 [0.875, 0.901] \\
\bottomrule\end{tabular}\end{table}
\subsection{SmolLM3: scoring 4096}
Recorded status: complete. The native automatic backend and MATH numerical reference retain separate executed losses.
\begin{table}[H]\centering\small
\caption{Actual executed policies, fixed fitting seed 901. Scoring NLL uses the complete locked roster; latency uses the first 30 rows. Median and p95 are whole-region milliseconds. Ratio is native automatic full median divided by policy median; values above one favor the policy. Block intervals quantify repeated timing under recorded shared-desktop conditions, not unseen devices or free-generation performance.}
\begin{tabular}{lrrrr}\toprule Policy & Median & p95 & Executed NLL & Native speed ratio [95\%]\\\midrule
native\_auto\_full & 1162.5 & 1255.6 & 2.3695 & 1.000 [1.000, 1.000] \\
matched\_math\_full & 1128.4 & 1171.6 & 2.3695 & 1.030 [1.002, 1.047] \\
true\_bypass & 1112.4 & 1158.4 & 2.3995 & 1.045 [1.015, 1.063] \\
random\_quota & 1120.4 & 1200.4 & 2.3935 & 1.038 [1.001, 1.057] \\
I\_T32 & 1118.0 & 1225.2 & 2.3930 & 1.040 [0.984, 1.059] \\
I\_S35 & 1114.9 & 1194.8 & 2.3787 & 1.043 [1.013, 1.060] \\
I\_ST32 & 1116.4 & 1264.6 & 2.3771 & 1.041 [1.002, 1.060] \\
C48 & 1183.6 & 1308.2 & 2.3923 & 0.982 [0.931, 1.000] \\
J32 & 1183.1 & 1261.2 & 2.3921 & 0.983 [0.955, 0.998] \\
S32\_T & 1184.1 & 1272.2 & 2.3843 & 0.982 [0.939, 0.997] \\
\bottomrule\end{tabular}\end{table}
\subsection{SmolLM3: scoring 512}
Recorded status: complete. The native automatic backend and MATH numerical reference retain separate executed losses.
\begin{table}[H]\centering\small
\caption{Actual executed policies, fixed fitting seed 901. Scoring NLL uses the complete locked roster; latency uses the first 30 rows. Median and p95 are whole-region milliseconds. Ratio is native automatic full median divided by policy median; values above one favor the policy. Block intervals quantify repeated timing under recorded shared-desktop conditions, not unseen devices or free-generation performance.}
\begin{tabular}{lrrrr}\toprule Policy & Median & p95 & Executed NLL & Native speed ratio [95\%]\\\midrule
native\_auto\_full & 61.3 & 68.0 & 2.5433 & 1.000 [1.000, 1.000] \\
matched\_math\_full & 61.4 & 68.7 & 2.5433 & 0.998 [0.961, 1.014] \\
true\_bypass & 60.9 & 69.1 & 2.5814 & 1.006 [0.991, 1.031] \\
random\_quota & 61.6 & 69.8 & 2.5739 & 0.994 [0.969, 1.006] \\
I\_T32 & 62.5 & 70.6 & 2.5678 & 0.980 [0.957, 1.002] \\
I\_S35 & 61.9 & 70.1 & 2.5506 & 0.989 [0.964, 1.006] \\
I\_ST32 & 62.6 & 69.6 & 2.5520 & 0.979 [0.948, 0.998] \\
C48 & 71.4 & 77.9 & 2.5651 & 0.858 [0.843, 0.883] \\
J32 & 71.0 & 81.2 & 2.5647 & 0.863 [0.846, 0.884] \\
S32\_T & 70.4 & 79.8 & 2.5512 & 0.870 [0.851, 0.890] \\
\bottomrule\end{tabular}\end{table}
\section{Completed measurement inventory}
All six scientific test banks completed after the all-family seal. Systems conditions attempted all ten planned policies. Failed policy checks are retained and are excluded from timing and quality claims; every valid condition uses thirty retained paired blocks. The original strict-telemetry CLI failed before timing under WDDM, followed by the documented pre-measurement platform adaptation.
\begin{table}[H]\centering\small
\caption{Realized systems coverage. Scoring quality uses 200 core rows at 512 or 99 midpoint rows at longer lengths; cached quality uses the same 30 timed source rows. No projection from dense branch losses substitutes for executed policy quality.}
\begin{tabular}{llrrl}\toprule Model & Condition & Valid policies & Blocks & Failed policies\\\midrule
SmolLM3 & scoring 512 & 10 & 30 & None \\
SmolLM3 & scoring 1024 & 9 & 30 & I\_S35 \\
SmolLM3 & scoring 4096 & 10 & 30 & None \\
SmolLM3 & cached 512 & 10 & 30 & None \\
SmolLM3 & cached 1024 & 10 & 30 & None \\
Qwen3.5 & scoring 512 & 10 & 30 & None \\
Qwen3.5 & scoring 1024 & 10 & 30 & None \\
Qwen3.5 & scoring 4096 & 10 & 30 & None \\
Qwen3.5 & cached 512 & 10 & 30 & None \\
Qwen3.5 & cached 1024 & 10 & 30 & None \\
\bottomrule\end{tabular}\end{table}

\end{document}